\documentclass{article}

\usepackage{arxiv}

\usepackage{algorithm}
\usepackage{algorithmic}
\usepackage[utf8]{inputenc} 
\usepackage[T1]{fontenc}    
\usepackage{hyperref}       
\usepackage{url}            
\usepackage{booktabs}       
\usepackage{amsfonts}       
\usepackage{nicefrac}       
\usepackage{microtype}      
\usepackage{xcolor}         
\usepackage{amsmath}
\usepackage{amssymb}
\usepackage{mathtools}
\usepackage{amsthm}
\usepackage{MnSymbol}

\usepackage{wrapfig}
\usepackage{subfigure}
\usepackage{enumitem}
\setitemize{noitemsep,topsep=0pt,parsep=0pt,partopsep=0pt}

\usepackage{bm}
\usepackage{multicol}
\usepackage{multirow}

\theoremstyle{plain}
\newtheorem{theorem}{Theorem}[section]
\newtheorem{proposition}[theorem]{Proposition}
\newtheorem{lemma}[theorem]{Lemma}

\theoremstyle{definition}
\newtheorem{definition}[theorem]{Definition}

\theoremstyle{remark}

\usepackage{xr}
\usepackage{subfiles}

\makeatletter
\newcommand*{\addFileDependency}[1]{
  \typeout{(#1)}
  \@addtofilelist{#1}
  \IfFileExists{#1}{}{\typeout{No file #1.}}
}
\makeatother

\newcommand*{\myexternaldocument}[1]{%
    \externaldocument{#1}%
    \addFileDependency{#1.tex}%
    \addFileDependency{#1.aux}%
}
\myexternaldocument{./neurips_appendix}

\title{Offline Reinforcement Learning with Causal Structured World Models}

\date{}

\author{%
  Zheng-Mao Zhu\textsuperscript{\rm 1}, 
    Xiong-Hui Chen\textsuperscript{\rm 1,3}, 
    Hong-Long Tian\textsuperscript{\rm 1}, 
    Kun Zhang\textsuperscript{\rm 2,5}, 
    Yang Yu \textsuperscript{\rm 1,3,4}\thanks{Yang Yu is the corresponding author}\\
  \textsuperscript{\rm 1} National Key Laboratory for Novel Software Technology, Nanjing University\\
  \textsuperscript{\rm 2} Department of Philosophy, Carnegie Mellon University\\
  \textsuperscript{\rm 3} Polixir.ai,   \textsuperscript{\rm 4} Peng Cheng Laboratory\\
  \textsuperscript{\rm 5} Department of Machine Learning, Mohamed bin Zayed University of Artificial Intelligence\\
  \texttt{\{zhuzm,chenxh,yuy\}@lamda.nju.edu.cn,franktian424@gmail.com,kunz1@cmu.edu}   \\
}

\begin{document}

\maketitle
\thispagestyle{plain}

\begin{abstract}
Model-based methods have recently shown promising for offline reinforcement learning (RL), aiming to learn good policies from historical data without interacting with the environment. Previous model-based offline RL methods learn fully connected nets as world-models to map the states and actions to the next-step states. However, it is sensible that a world-model should adhere to the underlying causal effect such that it will support learning an effective policy generalizing well in unseen states. In this paper, We first provide theoretical results that causal world-models can outperform plain world-models for offline RL by incorporating the causal structure into the generalization error bound. We then propose a practical algorithm, o\textbf{F}fline m\textbf{O}del-based reinforcement learning with \textbf{C}a\textbf{U}sal \textbf{S}tructure (FOCUS), to illustrate the feasibility of learning and leveraging causal structure in offline RL. Experimental results on two benchmarks show that FOCUS reconstructs the underlying causal structure accurately and robustly. Consequently, it performs better than the plain model-based offline RL algorithms and other causal model-based RL algorithms.
\end{abstract}

\section{Introduction}

Offline RL has gained great interest since it enables RL algorithms to scale to many real-world applications, e.g., recommender systems \cite{shi.aaai19,shang.mlj21}, autonomous driving \cite{DBLP:journals/corr/abs-1805-04687}, and healthcare \cite{2019Guidelines}, waiving costly online trial-and-error.
In the offline setting, Model-Based Reinforcement Learning (MBRL) is an important direction that already produces significant offline learning performance \cite{mopo,chen.nips21}. Moreover, learning models is also useful for training transferable policies \cite{luo.aaai22,chen.nips21}.
Therefore, there is increasing studies on learning world-models, from supervised prediction methods \cite{DBLP:conf/nips/HaS18} to adversarial learning methods \cite{DBLP:conf/nips/XuLY20}.

However, commonly there exists an underlying causal structure among the states and actions in various tasks.
The causal structure can support learning a policy with better generalization ability.
For example, in driving a car where the speed depends on the gas and brake pedals but not the wiper,
a plain world-model trained on rainy days always predicts deceleration when the wiper is turned on and thus can not generalize to other weather situations.
By contrast, a causal world-model can avoid the spurious dependence between wiper and deceleration (because of the rain) and hence help generalize well in unseen weather.
In fact, empirical evidence also indicates that inducing the causal structure is important to improve the generalization of RL \cite{DBLP:conf/cogsci/EdmondsKSZRZL18, DBLP:conf/atal/Tenenbaum18,DBLP:conf/iclr/BengioDRKLBGP20,DBLP:conf/nips/HaanJL19,zhu.aaai22}, but little attention was paid on causal world-model learning.

In this paper, we first provide theoretical support for the statement above: we show that a causal world-model can outperform a plain world-model for offline RL.
From the causal perspective, we divide the variables in states and actions into two categories, namely, causal variables and spurious variables, and then formalize the procedure that learns a world-model from raw measured variables.
Based on the formalization, we quantize the influence of the spurious dependence on the generalization error bound and thus prove that incorporating the causal structure can help reduce this bound.

We then propose a practical offline MBRL algorithm with causal structure, FOCUS, to illustrate the feasibility of learning causal structure in offline RL.
An essential step of FOCUS is to learn the causal structure from data and then use it properly. Learning causal structure from data has been known as causal discovery \cite{spirtes2000causation}.
There are some challenges in utilizing causal discovery methods in RL, and there are specific properties in the data that causal discovery may benefit from. 
Specifically, we extended the PC algorithm, which aims to find causal discovery based on the inferred conditional independence relations, to incorporate the constraint that the future cannot cause the past. Consequently, we can reduce the number of conditional independence tests and determine the causal direction.
We further adopt kernel-based independence testing \cite{kci}, which can be applied in continuous variables without assuming a functional form between the variables as well as the data distribution.

In summary, this paper makes the following key contributions:
\begin{itemize}[leftmargin=5mm]
\item It shows theoretically that a causal world-model outperforms a plain world-model in offline RL, in terms of the generalization error bound.
\item It proposes a practical algorithm, FOCUS, and illustrates the feasibility of learning and using a causal world-model for offline MBRL.
\item Our experimental results verify the theoretical claims, showing that FOCUS outperforms baseline models and other online causal MBRL algorithms in the offline setting.
\end{itemize}

\section{Related Work}
\textbf{Causal Structure Learning in Online MBRL.}
Despite that some methods have been proposed for learning causal structure in online RL, such methods all depend on interactions and do not have a mechanism to transfer into the offline setting.
The core of online causal structure learning is to evaluate the performance or other metrics of one structure by interactions and choose the best one as the causal structure.
\cite{DBLP:conf/nips/HaanJL19} parameterizes the causal structure in the model and learns policies for each possible causal structure by minimizing the log-likelihood of dynamics.
Given learned policies, it makes regression between the policy returns and the causal structure and then chooses the structure corresponding to the highest policy return.
\cite{DBLP:journals/corr/abs-1910-01075} (LNCM) samples causal structures from Multivariate Bernoulli distribution and scores those structures according to the log-likelihood on interventional data.
Based on the scores, it calculates the gradients for the parameters of the Multivariate Bernoulli distribution and updates the parameters iteratively.
\cite{DBLP:conf/iclr/BengioDRKLBGP20} utilizes the speed of adaptation to learn the causal direction, which does not provide a complete algorithm for learning causal structure.
By contrast, FOCUS utilizes a causal discovery method for causal structure learning, which only relies on the collected data to obtain the causal structure.

\textbf{Causal Discovery Methods.}
The data in RL is more complex than it in traditional causal discovery, where the data is often discrete and the causal relationship is under a simple linear assumption.
In recent years, practical methods have been proposed for causal discovery for continuous variables, which is the case we are concerned with in this paper.
\cite{DBLP:conf/icml/SunJSF07} is based on explicit estimation of the conditional densities or their variants, which exploit the difference between the characteristic functions of these conditional densities.
The estimation of the conditional densities or related quantities is difficult, which deteriorates the testing performance especially when the conditioning set is not small enough.
\cite{DBLP:conf/aaai/Margaritis05} discretizes the conditioning set to a set of bins and transforms conditional independence (CI) to the unconditional one in each bin. 
Inevitably, due to the curse of dimensionality, as the conditioning set becomes larger, the required sample size increases dramatically.
By contrast, the KCI test \cite{kci} is a popular and widely-used causal discovery method, in which the test statistic can be easily calculated from the kernel matrices and the distribution can be estimated conveniently.

\section{Preliminaries}

\textbf{Markov Decision Process (MDP).}
We describe the RL environment as an MDP with five-tuple $\langle \mathcal{S}, \mathcal{A}, P, R, \gamma \rangle$ \cite{bellman1957markovian}, where $\mathcal{S}$ is a finite set of states;
$\mathcal{A}$ is a finite set of actions; 
$P$ is the transition function with $P(\textbf{s}'|\textbf{s}, \textbf{a})$ denoting the next-state distribution after taking action $\textbf{a}$ in state $\textbf{s}$; 
$R$ is a reward function with $R(\textbf{s}, \textbf{a})$ denoting the expected immediate reward gained by taking action a in state s; 
and $\gamma \in [0, 1]$ is a discount factor. 
An agent chooses actions $\textbf{a}$ according to a policy $\textbf{a}\sim \pi(\textbf{s})$, which updates the system state $\textbf{s}'\sim P(\textbf{s},\textbf{a})$, yielding a reward $r\sim R(\textbf{s},\textbf{a})$. 
The agent's goal is to maximize the the expected cumulative return by learning a good policy $\max_{\pi,P} \mathbb{E}[\gamma^t R(\textbf{s}_t,\textbf{a}_t)]$. 
The state-action value $Q_\pi$ of a policy $\pi$ is the expected discounted reward of executing action $a$ from state $\textbf{s}$ and subsequently following policy $\pi$:
$Q_{\pi}(\textbf{s}, \textbf{a}):=R(\textbf{s}, \textbf{a})+\gamma \mathbb{E}_{\textbf{s}' \sim P, \textbf{a}' \sim \pi}\left[Q_{\pi}\left(\textbf{s}', \textbf{a}'\right)\right]$.

\textbf{Offline Model-based Reinforcement Learning.}
Model-based reinforcement learning typically involves learning a dynamics model of the environment by fitting it using a maximum-likelihood estimate of the trajectory-based data collected by running some exploratory policy \cite{DBLP:conf/icra/WilliamsWGDRBT17,DBLP:conf/iclr/KurutachCDTA18}. 
In the offline RL setting, where we only have access to the data collected by multiple policies, recent techniques build on the idea of pessimism (regularizing the original problem based on how confident the agent is about the learned model) and have resulted in better sample complexity over model-free methods on benchmark domains \cite{DBLP:conf/nips/KidambiRNJ20,mopo}. 

\section{Theory}

In this section, we provide theoretical evidence for the advantages of a causal world-model over a plain world-model, which shows that utilizing a good causal structure can reduce the generalization error bounds in offline RL.
Specifically, We incorporate the causal structure into the generalization error bounds, which include the model prediction error bound and policy evaluation error bound.
The full proof can be found in Appendix~\ref{app:theory}.
In this paper, we focus on the linear case.

\subsection{Model Prediction Error Bound}

In this subsection, we formulize the procedure that learns a plain world-model and then connect the model prediction error bound with the number of spurious variables in the plain world-model.
Specifically, we point out that the spurious variables lead the model learning problem to an ill-posed problem that has multiple optimal solutions, which consequently results in the increment of the model prediction error bound.
Since model learning can be viewed as a supervised learning problem, we provide the model prediction error bound in a supervised learning framework.

\textbf{Preliminary.}
Let $\mathcal{D}$ denote the data distribution where we have samples $(\textbf{X},Y)\sim \mathcal{D}, \textbf{X}\in\mathbb{R}^n$. 
The goal is to learn a linear function $f$ to predict $Y$ given $\textbf{X}$.
From the causal perspective, $Y$ is generated from only its causal parent variables rather than all the variables in $\textbf{X}$. 
Therefore we can split the variables in $\textbf{X}$ into two categories, $\textbf{X}=(\textbf{X}_{causal},\textbf{X}_{spurious})$:
\begin{itemize}[leftmargin=5mm]
    \item  $\textbf{X}_{causal}$ represents the causal parent variables of $Y$, that is, $Y=f^*(\textbf{X}_{causal})+\epsilon_{causal}$, where $f^*$ is the ground truth and $\epsilon_{causal}$ is a zero mean noise variable that $\textbf{X}_{causal} \upmodels \epsilon_{causal}$.
    \item $\textbf{X}_{spurious}$ represents the spurious variables that $\textbf{X}_{spurious}\upmodels \textbf{X}_{causal}$, but in some biased data sets $\textbf{X}_{spurious}$ and $\textbf{X}_{causal}$ have strong relatedness.
    In other words, $\textbf{X}_{spurious}$ can be predicted by $\textbf{X}_{causal}$ with small error, i.e., $\textbf{X}_{spurious}=\textbf{X}_{causal}\gamma_{spurious}+\epsilon_{spurious}$, where $\epsilon_{spurious}$ is the regression error with zero mean and small variance.
\end{itemize}

For clearly representation, we use $\textbf{X}_{cau}\triangleq\textbf{X}\circ\omega_{cau}$  ($\circ$ represents element-wise multiplication) to replace $\textbf{X}_{causal}$, where $cau$ records the indices of $\textbf{X}_{causal}$ in $\textbf{X}$ and  $(\omega_{cau})_i=\mathbb{I}(i\in cau)$. 
Correspondingly, we also use $\textbf{X}_{spu}\triangleq\textbf{X}\circ\omega_{spu}$ to replace $\textbf{X}_{spurious}$.
According to the definition of $\textbf{X}_{cau}$, we have $Y=(\textbf{X}\circ\omega_{cau})\beta^*+\epsilon_{cau}$, where $\omega_{cau}\circ\beta^*$ is the global optimal solution of the optimization problem
\begin{align}
\min_{\beta} \mathbb{E}_{(\textbf{X},Y)\sim\mathcal{D}}[\textbf{X}\beta-Y]^2.
\label{problem-0}
\end{align}
Above problem is easy if the data is uniformly sampled from $\mathcal{D}$.
However, in the offline setting, we only have biased data $\mathcal{D}_{train}$ sampled by given policy $\pi_{train}$, where the optimization objective is 
\begin{align}
    \min_{\beta} \mathbb{E}_{(\textbf{X},Y)\sim\mathcal{D}_{train}}[\textbf{X}\beta-Y]^2. 
    \label{problem-1}
\end{align}
The Problem~\ref{problem-1} has multiple optimal solutions due to the strong linear relatedness of $\textbf{X}_{spu}$ and $\textbf{X}_{cau}$ in $\mathcal{D}_{train}$, which is proved in Lemma~\ref{lemma-equal}.

\begin{lemma}
\label{lemma-equal}
Given that $\omega_{cau}\circ\beta^*$ is the optimal solution of Problem~\ref{problem-0},
suppose that in $D_{train}$, $\textbf{X}_{spu}=(\textbf{X}\circ\omega_{cau})\gamma_{spu}+\epsilon_{spu}$ where $\mathbb{E}_{D_{train}}[\epsilon_{spu}]=0$ and $\gamma_{spu}\neq\textbf{0}$, we have that $\hat{\beta}_{spu}\triangleq\omega_{cau}\circ(\beta^*-\lambda\gamma_{spu})+\lambda\omega_{spu}$ is also an optimal solution of Problem~\ref{problem-1} for any $\lambda$:
\begin{align}
    &\mathbb{E}_{(\textbf{X},Y)\sim D_{train}}\Big[\Big(|\textbf{X}(\omega_{cau}\circ\beta^*)-Y|_2\Big)\mid \textbf{X}\Big]
    =\mathbb{E}_{(\textbf{X},Y)\sim D_{train}}\Big[\Big(|\textbf{X}\hat{\beta}_{spu}-Y|_2\Big)\mid \textbf{X}\Big] \nonumber
\end{align}
\end{lemma}
The most popular method for solving such ill-posed problem is to add a regularization term for parameters $\beta$ \cite{DBLP:journals/corr/abs-1910-07113}:
\begin{align}
    \min_{\beta} \mathbb{E}_{(\textbf{X},Y)\sim\mathcal{D}_{train}}[\textbf{X}\beta-Y]^2+k\|\beta\|^2,
    \label{problem-2}
\end{align}
where $k$ is a coefficient. 
The form of Problem~\ref{problem-2} corresponds to the form of the ridge regression, which provides a closed-form solution of $k$ by Hoerl-Kennard formula \cite{10.2307/1271436}.

In the following, we will first introduce the solution of $\lambda$ under Problem~\ref{problem-2} in Lemma~\ref{lambda-lemma},
and then introduce the model prediction error bound with $\lambda$ in Theorem~\ref{theorem-spurious}.
For ease of understanding, we provide a simple version where the dimensions of $\textbf{X}_{cau}$ and $\textbf{X}_{spu}$ are both one ($|\textbf{X}_{cau}|=|\textbf{X}_{spu}|=1$).

\begin{lemma}[\textbf{$\lambda$ Lemma}]
\label{lambda-lemma}
Given $\lambda$ as the coefficient in Lemma~\ref{lemma-equal}, and $k$ in Problem~\ref{problem-2} chosen by Hoerl-Kennard formula, we have the solution of $\lambda$ in Problem~\ref{problem-2} that:
\begin{align}
\lambda
&=\frac{ \beta^*\gamma_{spu}}{{\beta^*}^2+ \gamma_{spu}^{2}+1+\frac{\sigma ^{2}_{spu}}{\sigma^2_{cau}}(1+\frac{1}{(\beta^*)^2})}
\label{formula-lambda}
\end{align}
\end{lemma}

Based on Lemma~\ref{lambda-lemma}, we can find that the smaller  $\sigma^2_{spu}$ (it means that $\textbf{X}_{spu}$ and $\textbf{X}_{cau}$ have stronger relatedness in the training dataset $D_{train}$), the larger $\lambda$. 
And we also have its bound:
\begin{proposition}
Given $\lambda$ as Formula~\ref{formula-lambda}, the bound of $\lambda$ is that
$
-\frac{1}{2}\leq\lambda\leq\frac{1}{2}.
$
\end{proposition}

\begin{theorem}[\textbf{Spurious Theorem}]
\label{theorem-spurious}
Let $\mathcal{D}=\{(\textbf{X},Y)\}$ denote the data distribution, $\hat{\beta}_{spu}$ denote the solution in Lemma~\ref{lemma-equal} with $\lambda$ in Lemma~\ref{lambda-lemma},
and $\hat{Y}_{spu}=\textbf{X}\hat{\beta}_{spu}$ denote the prediction.
Suppose that the data value is bounded: $|X_i|_1 \leq X_{max}, i=1,\cdots,n$ and the error of optimal solution $\epsilon_{cau}$ is also bounded: $|\epsilon_{cau}|_1\leq\epsilon_c$, we have the model prediction error bound:
\begin{align}
\mathbb{E}_{(\textbf{X},Y)\sim D}[(|\hat{Y}_{spu}-Y|_1)\mid \textbf{X}]\leq X_{max}|\lambda|_1(|\gamma_{spu}|_1+1)+\epsilon_{c}.
\end{align}
\end{theorem}
Theorem~\ref{theorem-spurious} shows that 
\begin{itemize}[leftmargin=5mm]
    \item The upper bound of the model prediction error $|\hat{Y}_{spu}-Y|_1$ increases by $X_{max}|\lambda|_1(|\gamma_{spu}|_1+1)$ for each induced spurious variable $X_{spu}$ in the model.
    \item When $X_{spu}$ and $X_{cau}$ have stronger relatedness (which means a bigger $\lambda$), the increment of the prediction model error bound led by $X_{spu}$ is bigger.
\end{itemize}

\subsection{Policy Evaluation Error Bound}

Although in most cases, an accurate model ensures a good performance in MBRL, the model error bound is still an indirect evaluation compared to the policy evaluation error bound for MBRL.
In this subsection, we apply the spurious theorem (Theorem~\ref{theorem-spurious}) to offline MBRL and provide the policy evaluation error bound with the number of spurious variables.

Suppose that the state value and reward are bounded that $|S_{t,i}|_1\leq S_{max}, R_t\leq R_{max}$, let $\lambda_{max}$ denote the maximum of $\lambda$ and $\gamma_{max}$ denote the maximum of $|\gamma_{spu}|_1$, we have the policy evaluation error bound in Theorem~\ref{theorem-spurious-rl}.
\begin{theorem}[\textbf{RL Spurious Theorem}]
Given an MDP with the state dimension $n_s$ and the action dimension $n_a$, 
a data-collecting policy $\pi_D$,
let $M^*$ denote the true transition model, 
$M_\theta$ denote the learned model that $M_\theta^i$ predicts the $i^{th}$ dimension with spurious variable sets $spu_i$ and causal variables $cau_i$, i.e.,  $\hat{S}_{t+1,i}=M^i_{\theta}((\textbf{S}_t,\textbf{A}_t)\circ \omega_{cau_i\cup spu_i})$.
Let $V_{\pi}^{M_\theta}$ denote the policy value of the policy $\pi$ in model $M_{\theta}$ and correspondingly $V_{\pi}^{M^*}$.
For an arbitrary bounded divergence policy $\pi$, i.e. $\max_{S} D_{KL}(\pi(\cdot|S),\pi_{D}(\cdot|S))\leq \epsilon_{\pi}$, we have the policy evaluation error bound:
\begin{align}
|V_{\pi}^{M_\theta}-V_{\pi}^{M^*}|\leq 
\frac{2\sqrt{2}R_{max}}{(1-\gamma)^2}\sqrt{\epsilon_{\pi}}+ \nonumber
\frac{R_{max}\gamma}{2(1-\gamma)^2}S_{max}[n_s\epsilon_c+\nonumber 
(1+\gamma_{max})\lambda_{max}n_s(n_s+n_a)R_{spu}]
\end{align}
where $R_{spu}=\frac{\sum_{i=1}^{n_s}|spu_i|}{n_s(n_s+n_a)}$, which represents the spurious variable density, that is, the ratio of spurious variables in all input variables .
\label{theorem-spurious-rl}
\end{theorem}
Theorem~\ref{theorem-spurious-rl} shows the relation between the policy evaluation error bound and the spurious variable density, which indicates that:
\begin{itemize}[leftmargin=5mm]
    \item When we use a non-causal model that all the spurious variables are input, $R_{spu}$ reaches its maximum value $\bar{R}_{spu}< 1$. By contrast, in the optimal causal structure, $R_{spu}$ reaches its minimum value of $0$.
    \item The density of spurious variables $R_{spu}$ and the correlation strength of spurious variables $\lambda_{max}$ both influence the policy evaluation error bound.
    However, if we exclude all the spurious variables, i.e., $R_{spu}=0$, the correlation strength of spurious variables will have no effect.
\end{itemize}

\section{Algorithm}

\begin{figure}[t!]
\begin{center}
\centerline{\includegraphics[width=0.95\linewidth]{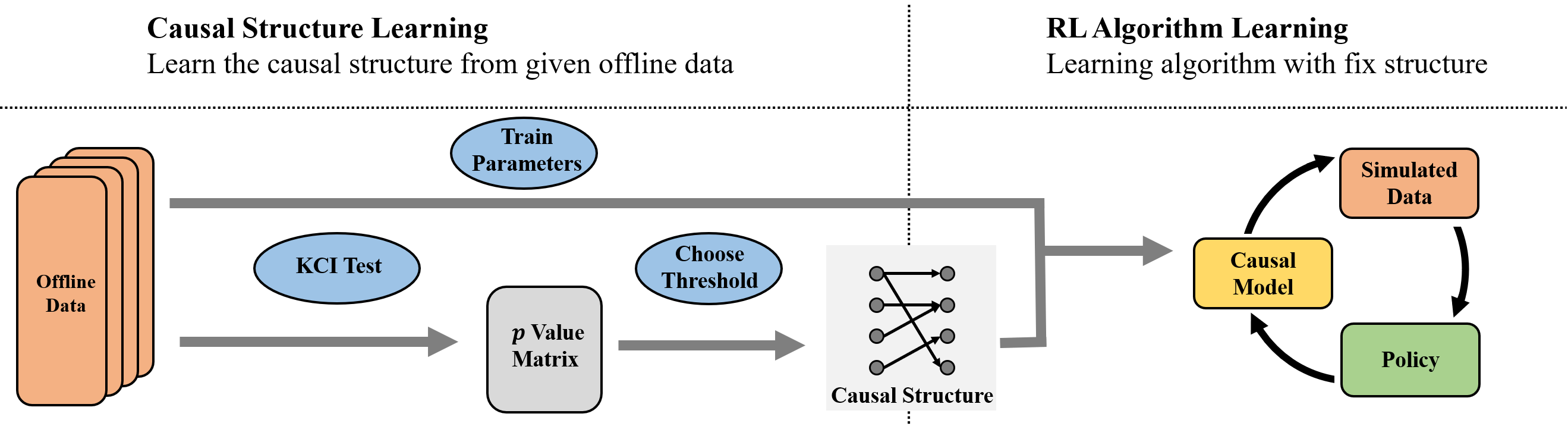}}
\caption{The architecture of FOCUS. Given offline data, FOCUS learns a $p$ value matrix by KCI test and then gets the causal structure by choosing a $p$ threshold. After combining the learned causal structure with the neural network, FOCUS learns the policy through an offline MBRL algorithm.}
\vskip -0.2in
\label{fig-architecture}
\end{center}
\vskip -0.1in
\end{figure}

In the theory section, we have provided the theoretical evidence about the advantages of a causal world-model over a plain world-model.
Besides lower prediction errors, a causal world-model also matters for better decision-making in RL.
In the condition that spurious variables do not increase prediction errors (e.g., spurious variables disturb only in unreachable states), a wrong causal relation also leads to terrible decision-making.
For example, rooster crowing can predict the rise of the sun, but forcing a rooster to crow for a sunny day is a natural decision if we have a wrong causal relation that rooster crowing causes the rise of the sun.
In the above example, predicting the rise of the sun by rooster crowing is a zero-error world-model since rooster crowing on a rainy day is an unreachable state, but such world-model leads to terrible decision-making.

After demonstrating the necessity of a causal world-model in offline RL, in this section we propose a practical offline MBRL algorithm, FOCUS, to illustrate the feasibility of learning causal structure in offline RL.
The main idea of FOCUS is to take the advantage of causal discovery methods and extend it to offline MBRL.
Compared to previous online causal structure learning methods, the causal discovery method brings the following advantages in the offline setting:
\begin{itemize}[leftmargin=5mm]
    \item \textbf{Robust:} The structure is not influenced by the performance of the test data, which is artificially selected and biased.
    \item \textbf{Efficient:} The causal discovery method directly returns the structure by independence testing without any network training procedure and thus saves the samples for network training.
\end{itemize}

\subsection{Preliminary}

\begin{wrapfigure}{r}{0.5\columnwidth}
\begin{center}
\vskip -0.2in
\begin{minipage}[b]{0.15\columnwidth}
\centerline{\includegraphics[width=0.9\columnwidth]{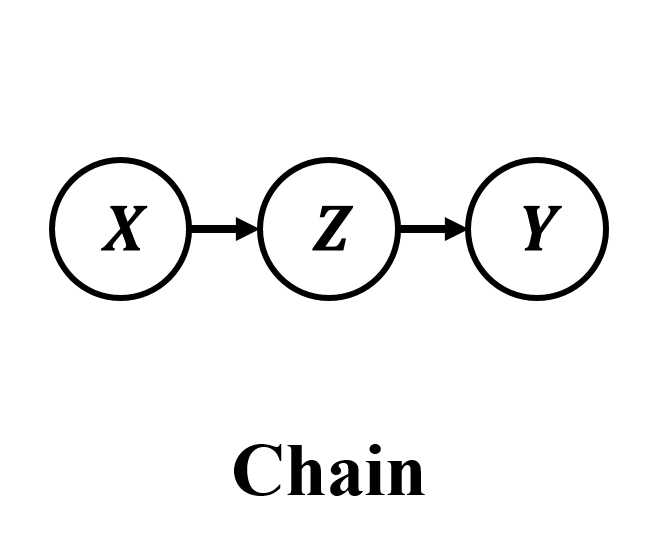}}
\end{minipage}
\begin{minipage}[b]{0.15\columnwidth}
\centerline{\includegraphics[width=0.9\columnwidth]{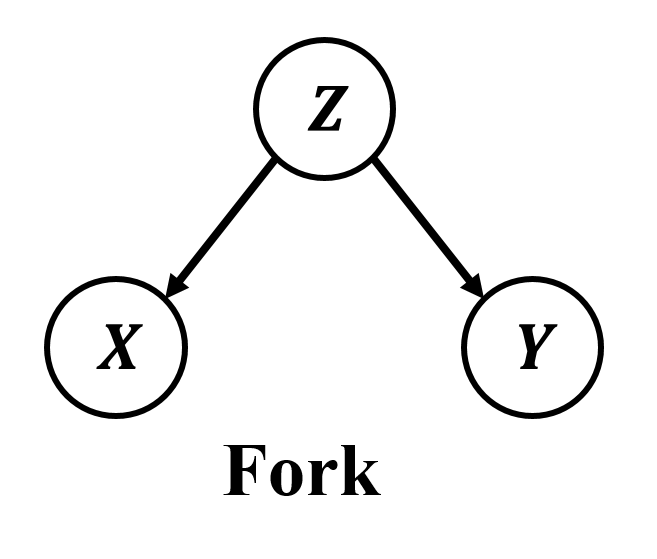}}
\end{minipage}
\begin{minipage}[b]{0.15\columnwidth}
\centerline{\includegraphics[width=0.9\columnwidth]{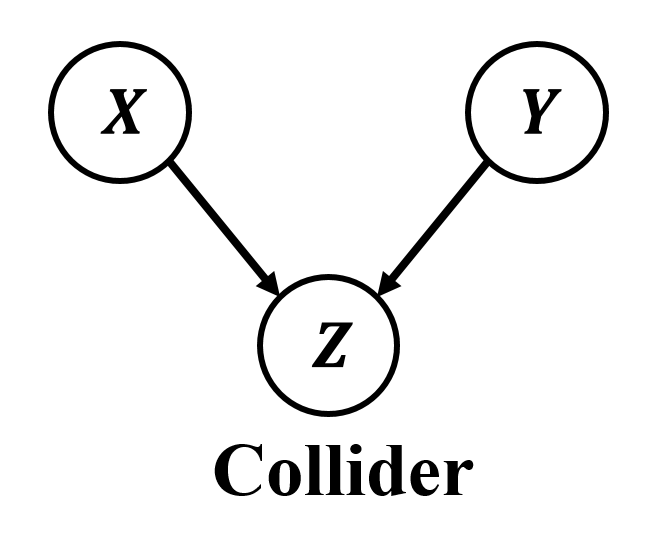}}
\end{minipage}
\caption{The three basic structure for three variables $X, Y$ and $Z$, where $Z$ plays an important role in causal discovery.
}
\label{fig-chain}
\end{center}
\vskip -0.2in
\end{wrapfigure}
\textbf{Conditional Independence Test.}
Independence and conditional independence (CI) play a central role in causal discovery \cite{DBLP:journals/ijon/Shanmugam01,DBLP:journals/technometrics/Burr03,DBLP:journals/ker/Parsons11a}. 
Generally speaking, the CI relationship $X\upmodels Y\mid Z$ allows us to drop $Y$ when constructing a probabilistic model for $X$ with $(Y, Z)$.
There are multiple CI testing methods for various conditions, which provide the correct conclusion only given the corresponding condition.
The kernel-based Conditional Independence test (KCI-test) \cite{kci} is proposed for continuous variables without assuming a functional form between the variables as well as the data distributions, which is the case we are concerned with in this paper.
Generally, the hypothesis $H_0$ that variables are conditionally independent is rejected when $p$ is smaller than the pre-assigned significance level, say, 0,05. In practice, we can design the significance level instead of a fixed value.

\textbf{Conditional Variables.}
Besides the specific CI test method, the conclusion of conditional independence testing also depends on the conditional variable $Z$, that is, different conditional variables can lead to different conclusions.
Taking the triple $(X,Y,Z)$ as an example, there are three typical structures, namely, \textit{Chain}, \textit{Fork}, and \textit{Collider} as shown in Fig~\ref{fig-chain}, where whether conditioning on $Z$ significantly influences the testing conclusion.

\begin{itemize}[leftmargin=5mm]
    \item \textit{Chain:} There exists causation between $X$ and $Y$ but conditioning on $Z$ leads to independence.
    \item \textit{Fork:} There does not exist causation between $X$ and $Y$ but not conditioning on $Z$ leads to non-independence.
    \item \textit{Collider:} There does not exist causation between $X$ and $Y$ but conditioning on $Z$ leads to non-independence.
\end{itemize}

\subsection{Building the Causal Structure from the KCI test}

\textbf{Applying the Independence Test in RL.}
Based on the preliminaries, given the two target variables $X, Y$ and the condition variable $Z$, the KCI test returns a probability value $p=f_{KCI}(X, Y, Z)\in[0,1]$, which measures the probability that $X$ and $Y$ are conditionally independent given the condition $Z$.
In other words, a small $p$ implies that $X$ and $Y$ have causation given $Z$.
To transform an implicit probability value into an explicit conclusion of whether the causation exists, we design a threshold $p^*$ that:
\begin{align*}
    Causation(X,Y)=\mathbb{I}(f_{KCI}(X,Y,Z)\leq p^*)\in\{0,1\},
\end{align*}
where $Causation(X,Y)=1$ represents independence and $0$ represents that causation exists. 
Details of choosing $p^*$ can be found in Appendix~\ref{app:alg-1}.

In model learning of RL, variables are composed of states and actions of the current and next timesteps and the causal structure refers to whether a variable in $t$ timestep (e.g., the $i^{th}$ dimension, $X_t^i$) causes another variable in $t+1$ timestep (e.g., the $j^{th}$ dimension, $X_{t+1}^j$).
With the KCI test, we get the causal relation through the function $Causation(\cdot,\cdot)$ for each variable pair $(X_t^i,X_{t+1}^j)$ and then form the causal structure matrix $\mathcal{G}$:
\begin{align*}
    \mathcal{G}_{i,j}=Causation(X_t^i,X_{t+1}^j),
\end{align*}
where $\mathcal{G}_{i,j}$ is the element in row $i$ and column $j$ of $\mathcal{G}$.

\begin{figure*}
\begin{center}
\centerline{\includegraphics[width=0.6\columnwidth]{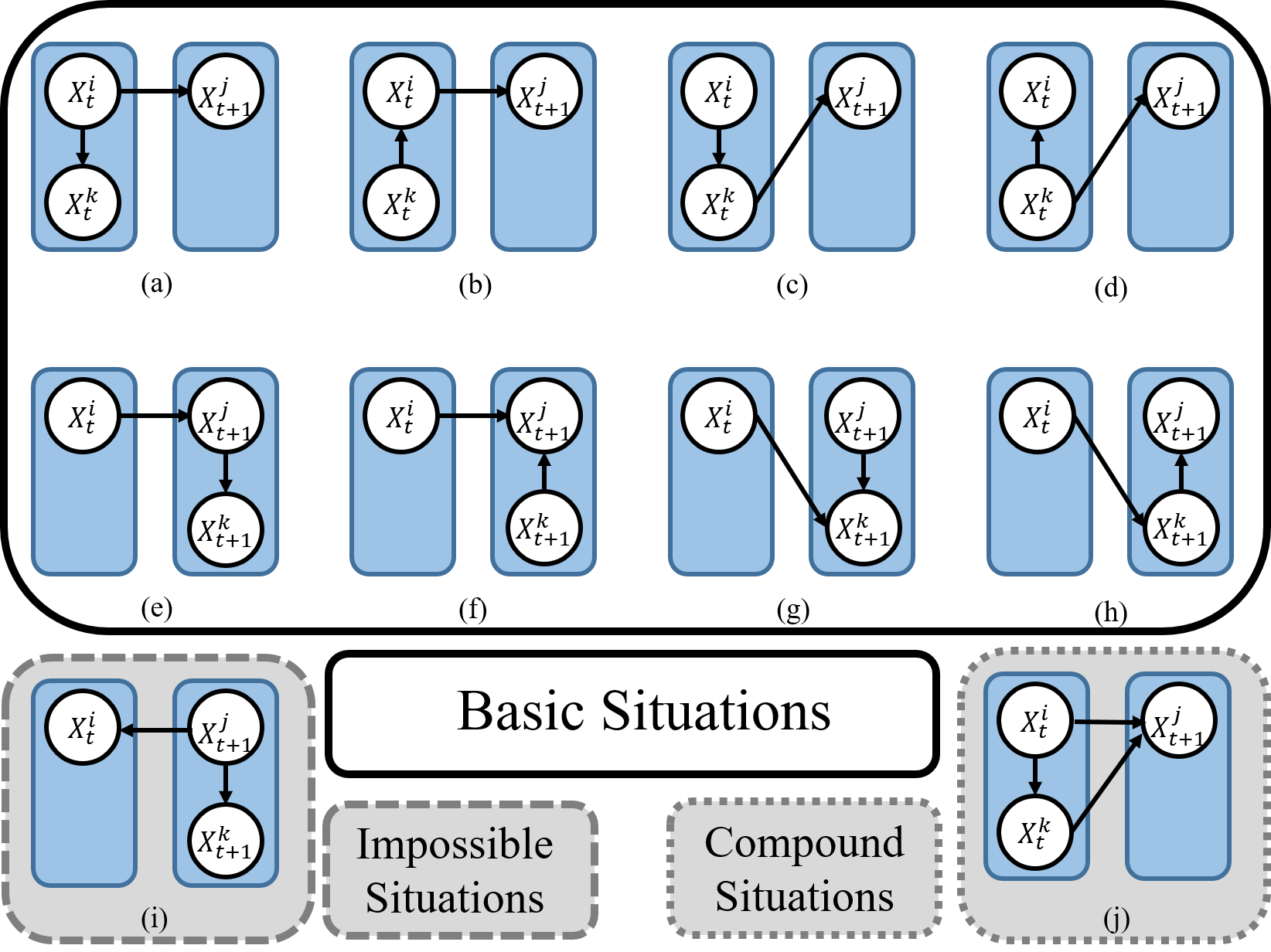}}
\caption{The basic, impossible and compound situations of the causation between target variables and condition variables.
In the basic situations,
\textbf{Top Line: }(a)-(d) list the situations that the condition variable $X^k$ is in the $t$ time step. 
\textbf{Bottom Line: }Similarly, (e)-(h) list the situations that the condition variable $X^k$ is in the $t+1$ time step.
}
\label{fig-causal-XYZ-new}
\end{center}
\end{figure*}

\textbf{Choosing the Conditional Variable in RL.}
As we said in preliminaries, improper conditional variables can reverse the conclusion of independence testing.
Therefore we have to carefully design the conditional variable set, which should include the intermediate variable of \textit{Chain}, the common parent variable of \textit{Fork}, but not the common son variable of \textit{Collider}.
Traditionally, the independence test has to traverse all possibilities of the conditional variable set and gives the conclusion, which is too time-consuming.
However, in RL we can reduce the number of conditional independence tests by incorporating the constraint that the future cannot cause the past.
Actually, this constraint limits the number of possible conditional variable set to a small value.
Therefore we can even take a classified discussion for each possible conditional variable set.
Before the discussion, we exclude two kinds of situations for simplicity:
\begin{itemize}[leftmargin=5mm]
    \item \textbf{Impossible situations.} We exclude some impossible situations as Fig~\ref{fig-causal-XYZ-new} (i) by the temporal property of data in RL.
    Specifically, the direction of the causation cannot be from the variable of $t+1$ time step to that of $t$ time step because the effect cannot happen before the cause.
    \item \textbf{Compound situations.} We only discuss the basic situations and exclude the compound situations, e.g., Fig~\ref{fig-causal-XYZ-new} (j), which is a compound of (a) and (c).
    It is because in such compound situations, the target variables $X_t^i$ and $X_{t+1}^j$ have direct causation (or it can not be a compound situation) and the independence testing only misjudges independence as non-independence but not non-independence as independence.
\end{itemize}

We list all possible situations of target variables $X_t^i,X_{t+1}^j$ and condition variable $X_{t/t+1}^k$ in the world-model as shown in Fig~\ref{fig-causal-XYZ-new}.
Based on the causal discovery knowledge in the preliminaries, we analyze basic situations in the following:
\begin{itemize}[leftmargin=5mm]
    \item \textbf{Top Line:} In (a)(b), whether conditioning on $X_t^k$ does not influence the conclusion of causation;
    In (c), although $X_t^k$ plays an intermediate variable in a \textit{Chain} and conditioning on $X_t^k$ leads to the conclusion of independence of $X_t^i$ and $X_{t+1}^j$, the causal parent set of $X_{t+1}^j$ will include $X_t^k$ when testing the causal relation between $X_t^k$ and $X_{t+1}j$, which can offset the influence of excluding $X_t^i$.
    In (d), conditioning on $Z$ is necessary for getting the correct conclusion of causation since $X_t^k$ is the common causal parent in a $\textit{Fork}$ structure. 
    \item \textbf{Bottom Line:} In (e)(f), whether conditioning on $X_{t+1}^k$ does not influence the conclusion of causation;
    In (g), not conditioning on $X_{t+1}^k$ is necessary to get the correct conclusion of causation since $X_{t+1}^k$ is the common son in a $\textit{Collider}$ structure;
    In (h), although $X_{t+1}^k$ plays an intermediate variable in a \textit{Chain} and not conditioning on $X_{t+1}^k$ leads to the conclusion of non-independence of $X_t^i$ and $X_{t+1}^j$, including $X_t^i$ in the causal parent set of $X_{t+1}^j$ will not induce any problem since $X_t^i$ does indirectly cause $X_{t+1}^j$.
\end{itemize}
Based on the above classified discussion, we can conclude our principle for choosing conditional variables in RL that: (1) Condition on the other variables in $t$ time step; (2) Do not condition on the other variables in $t+1$ time step.

\subsection{Combining Learned Causal Structure with An Offline MBRL Algorithm}

We combine the learned causal structure with an offline MBRL algorithm, MOPO \cite{mopo}, to form a causal offline MBRL algorithm as in Fig~\ref{fig-architecture}.
The complete learning procedure is shown in Algorithm~\ref{alg-model}, where Algorithm~\ref{alg-network} can be found in Appendix~\ref{app:alg-2}. 
Notice that our causal model learning method could be combined with any offline MBRL algorithm in principle.
More implementation details and hyperparameter values are summarized in Appendix~\ref{app:alg-1}.

\begin{algorithm}[t!]
  \caption{Causal Model Framework for Offline MBRL}
  \label{alg-model}
\begin{algorithmic}
  \STATE {\bfseries Input:} offline data set $\mathcal{D}=\{(\textbf{s}_t,\textbf{a}_t,\textbf{s}_{t+1},r_t)\}$;
  model $\mathcal{M}(\cdot;\theta)$;
  
  \STATE {\bf Stage 1:} Causal Structure Learning
  \STATE Get $p$ value matrix $G_{p}$ by KCI testing.
  \STATE Get the threshold $p^*$ by $G_{p}$.
  \STATE Get causal structure mask matrix $G$ by the threshold $p^*$.
  \STATE {\bf Stage 2:} Offline Reinforcement Learning
  \STATE {Choose an offline model-based reinforcement learning algorithm {\bfseries Algo($\cdot$)} and replace its model $\mathcal{M}(\cdot)$ by $\mathcal{M}_{Causal}(\cdot,G,\mathcal{M})$ (Algorithm~\ref{alg-network} in Appendix).}
  \STATE {Obtain the optimal policy $\pi^*=\textbf{Algo}(\mathcal{D})$.}
  \STATE {\bfseries Return} $\pi^*$
\end{algorithmic}
\end{algorithm}

\section{Experiments}

To demonstrate that 
(1) Learning a causal world-model is feasible in offline RL 
and (2) a causal world-model can outperform a plain world-model and other related online methods in offline RL, we evaluate (1) \textbf{causal structure learning} and (2) \textbf{policy learning} on the Toy Car Driving and MuJoCo benchmark.
Toy Car Driving is a simple and typical environment that is convenient to evaluate the accuracy of learned causal structure because We can design the causation between variables in it.
The MuJoCo is the most common benchmark to investigate the performance in continuous controlling, where each dimension of the state has a specific meaning and is highly abstract.
We evaluate FOCUS on the following indexes: 
(1) The \textit{accuracy}, \textit{efficiency} and \textit{robustness} of causal structure learning.
(2) The \textit{policy return} and \textit{generalization ability} in offline MBRL.

\textbf{Baselines.} 
We compare FOCUS with the sota offline MBRL algorithm, MOPO, and other online RL algorithms that also learn causal structure in two aspects, causal structure learning and policy learning.
(1) MOPO \cite{mopo} is a well-known and widely-used offline MBRL algorithm, which outperforms standard model-based RL algorithms and prior sota model-free offline RL algorithms on existing offline RL benchmarks.
The main idea of MOPO is to artificially penalize rewards by the uncertainty of the dynamics, which can avoid the distributional shift issue.
MOPO can be seen as the blank control with a plain world-model.
(2) Learning Neural Causal Models from Unknown Interventions (LNCM) \cite{DBLP:journals/corr/abs-1910-01075} is an online MBRL, in which the causal structure learning method can be transformed to the offline setting with a simple adjustment.
We take LNCM as an example to show that an online method cannot be directly transferred into offline RL.

\textbf{Environment.} 
\textbf{Toy Car Driving.}
Toy Car driving is a typical RL environment where the agent can control its direction and velocity to finish various tasks including avoiding obstacles and navigating. 
The information of the car, e.g., position, velocity, direction, and acceleration, can form the state and action in an MDP.
In this paper, we use a 2D Toy Car driving as the RL environment where the task of the car is to arrive at the destination (The visualization of states and a detailed description can be found in Appendix~\ref{app:exp-1}).
The state includes the direction $d$, the velocity (scalar) $v$, the velocity on the $x$-axis (one dimensional vector) $v_x$, the velocity on the $y$-axis $v_y$ and the position $(p_x,p_y)$.
The action is the steering angle $a$.
The visualization of the causal graph can be found in Appendix~\ref{app:exp-1}.
This causal structure is designed to demonstrate how a variable become spurious for others and highlight their influence in model learning.
\begin{figure*}
\begin{center}
\centerline{\includegraphics[width=0.3\columnwidth]{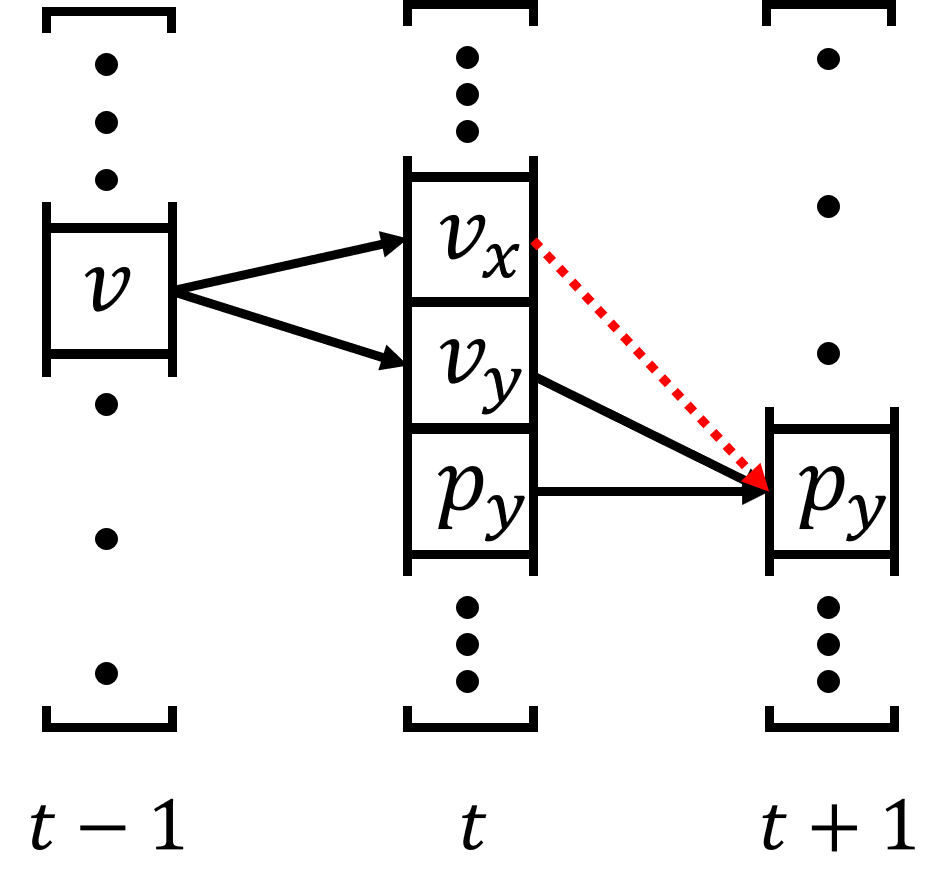}}
\caption{The visualization of the example.
The red dotted arrow presents that $(v_x)_t$ is a spurious variable for $(p_y)_{t+1}$.
}
\label{fig-spurious-example}
\end{center}
\end{figure*}

For example, when the velocity $v_{t-1}$ maintains stationary due to an imperfect sample policy, $(v_x)_t$ and $(v_y)_t$ have strong relatedness that $(v_x)_t^2+(v_y)_t^2=v_{t-1}^2$ and one can represent the other.
Since we design that $(p_y)_{t+1}-(p_y)_{t}=(v_y)_t$, $(v_x)_t$ and $(p_y)_{t+1}-(p_y)_{t}$ also have strong relatedness, which leads to that $(v_x)_t$ becomes a spurious variable of $(p_y)_{t+1}$ given $(p_y)_{t}$, despite that $(v_x)_t$ is not the causal parent of $y_{t+1}$.
By contrast, when the data is uniformly sampled with various velocities, this spuriousness will not exist.
\textbf{MuJoCo.}
MuJoCo \cite{mujoco} is a general-purpose physics engine, which is also a well-known RL environment.
MuJoCo includes multijoint dynamics with contact, where the variables of the state represent the positions, angles, and velocity of the agent. 
The dimensions of the state are from 3 to dozens.
The limited dimensions and the clear meaning of each variable provide the convenience of causal structure learning.

\textbf{Offline Data.}
We prepare three offline data sets, \textit{Random}, \textit{Medium}, and \textit{Replay} for the Car Driving and MuJoCo. 
\textit{Random} represents that data is collected by random policies.
\textit{Medium} represents that data is collected by a fixed but not well-trained policy, which is the least diverse.
\textit{Medium-Replay} is a collection of data that is sampled during training the \textit{Medium} policy, which is the most diverse.
The heat map of the data diversity is shown in Appendix~\ref{app:exp-1}.

\vskip -0.1in
\begin{table*}[ht]
\caption{The results on causal structure learning of our model and the baselines. Both the accuracy and the variance are calculated by five times experiments. \textit{FOCUS (-KCI)} represents FOCUS with a linear independence test. \textit{FOCUS (-CONDITION)} represents FOCUS with choosing all other variables as conditional variables.}
\label{table-cd}
\vskip -0.2in
\begin{center}
\begin{small}
\begin{sc}
\begin{tabular}{lcccc}
\toprule
Index         & FOCUS & LNCM & FOCUS(-kci)& FOCUS(-condition)  \\
\midrule
Accuracy      & \textbf{0.993}           & 0.52 & 0.62 & 0.65            \\ 
Robustness    & \textbf{0.001}           & 0.025  & 0.173 & 0.212              \\ 
Efficiency(Samples)    & $\textbf{1}\times\textbf{10}^\textbf{6}$           & $1\times10^7$ & $1\times10^6$ & $1\times10^6$                 \\ 
\bottomrule
\end{tabular}
\end{sc}
\end{small}
\end{center}
\vskip -0.1in
\end{table*}

\subsection{Causal Structure Learning} 

We compare FOCUS with baselines on the causal structure learning with the indexes of the \textit{accuracy}, \textit{efficiency}, and \textit{robustness}.
The accuracy is evaluated by viewing the structure learning as a classification problem, where causation represents the positive example and independence represents the negative example. 
The efficiency is evaluated by measuring the samples for getting a stable structure.
The robustness is evaluated by calculating the variance in multiple experiments.
The results in Table~\ref{table-cd} show that FOCUS surpasses LNCM on accuracy, robustness, and efficiency in causal structure learning. 
Noticed that LNCM also has a low variance because it predicts the probability of existing causation between any variable pairs with around $50\%$, which means that its robustness is meaningless.

\begin{table}[t!]
\caption{The comparison on converged policy return in the two benchmarks.
The detailed training curves are in Appendix~\ref{app:exp-1}.}
\vskip -0.1in
\label{table-reb-1}
\begin{center}
\begin{small}
\begin{sc}
\begin{tabular}{l|ccc|ccc}
\toprule
\multicolumn{1}{c|}{Env} & \multicolumn{3}{c|}{Car Driving} & \multicolumn{3}{c}{MuJoCo(Inverted Pendulum)} \\  
\midrule
                & \multicolumn{1}{c}{Random}   & \multicolumn{1}{c}{Medium}   & \multicolumn{1}{c|}{Replay}  & \multicolumn{1}{c}{Random}        & \multicolumn{1}{c}{Medium}      & \multicolumn{1}{c}{Replay}        \\  
\midrule
FOCUS             & \bm{$68.1 \pm 20.9$}    & $-58.9 \pm 41.3$         & \bm{$86.2\pm 18.2$} & \bm{$23.5\pm 17.9$}      & \bm{$24.9\pm 14.1$}  &  \bm{$49.2\pm 19.0$}          \\
\midrule
MOPO              & $-30.3 \pm 49.9$        & $-50.1 \pm 34.2$         & $46.2 \pm 28.1$     &      $8.5 \pm 6.2$  & $2.5\pm 0.08$   & $43.4 \pm 7.7$ \\
\midrule
LNCM              & $9.9 \pm 42.5$          &\bm{$-5.4 \pm 32.5$}      & $11.4 \pm 24.0$       &    $13.3 \pm 0.9$         &    $3.1 \pm 0.7$          &  $16.3 \pm 6.4$        \\
\bottomrule    
\end{tabular}
\end{sc}
\end{small}
\end{center}
\vskip -0.2in
\end{table}

\subsection{Policy Learning}

\textbf{Policy Return.}
We evaluate the performance of FOCUS and baselines in the two benchmarks on three different and typical offline data sets.
The results in Table~\ref{table-reb-1} show that FOCUS outperforms baselines by a significant margin in most data sets.
In \textit{Random}, FOCUS has the most significant performance gains to the baselines in both benchmarks because of the accuracy of causal structure learning in FOCUS. 
By contrast, in \textit{Medium-Replay}, the performance gains of FOCUS are least since the high data diversity in \textit{Medium-Replay} leads to weak relatedness of spurious variables (corresponds to small $\lambda$), which verifies our theory.
In \textit{Medium}, the results in the two benchmarks are different. In Car Driving, the relatively high score of LNCM does not mean that LNCM is the best but all three fail. 
The failure indicates that extremely biased data makes even the causal model fail to generalize.
However, the success of FOCUS in the Inverted Pendulum indicates that causal world-models depend less on the data diversity since FOCUS can still reach high scores in such a biased dataset where the baselines fail.
Here we only provide the results in \textit{Inverted Pendulum} but not all the environments in MuJoCo due to the characteristics of the robot control, specifically the frequency of observations, which we present a detailed description in Appendix~\ref{app:exp-1}.

\begin{figure*}
\begin{center}
\centerline{\includegraphics[width=0.4\columnwidth]{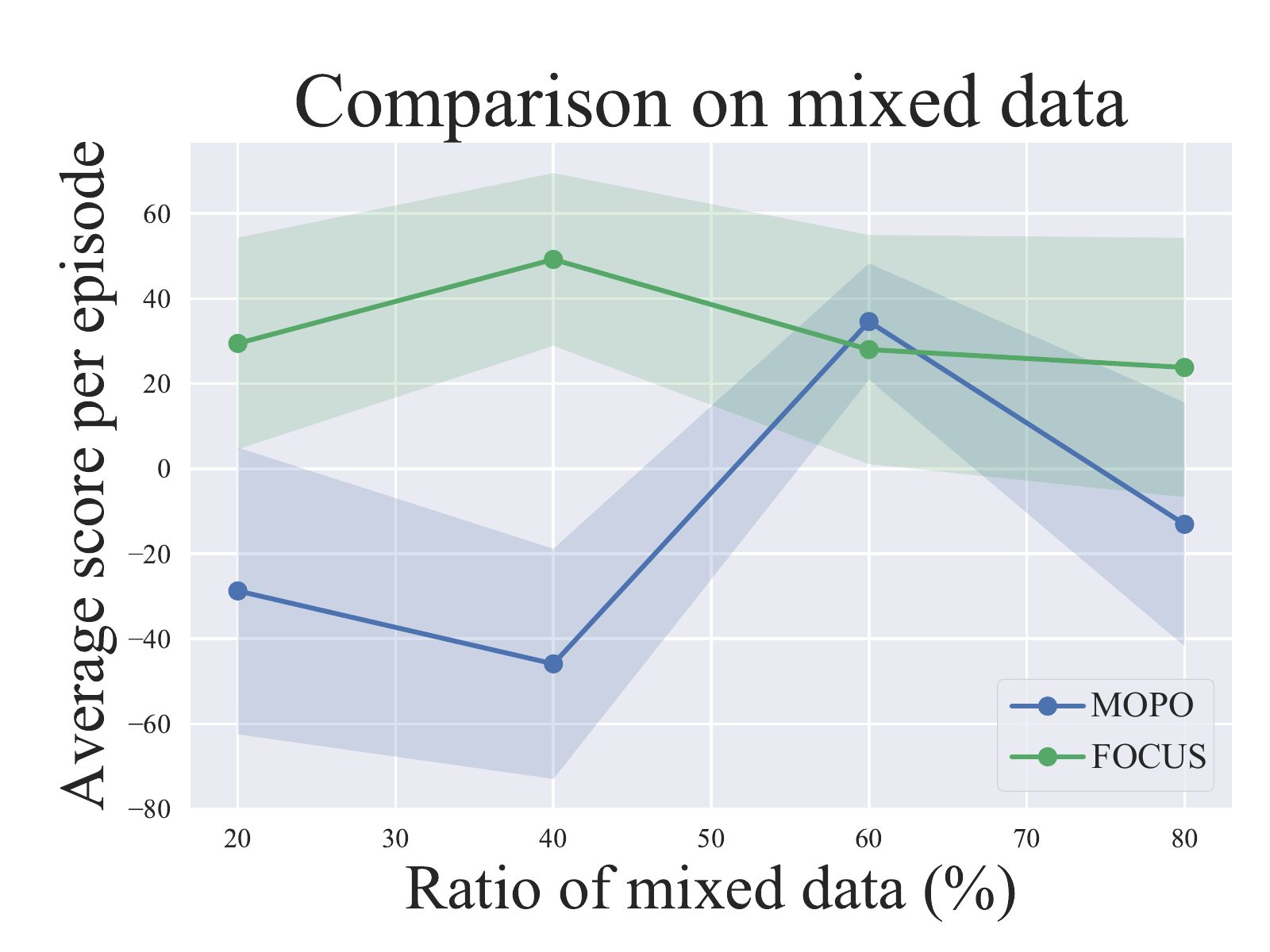}}
\caption{The comparison on generalization in mixed data.}
\label{fig-mix-data}
\end{center}
\end{figure*}
\textbf{Generalization Ability.} We compare the performance on different offline data sets, which is produced by mixing up \textit{Medium-Replay} and \textit{Medium} with different ratios.  
The $X\%$ in the x-axis represents that the data is mixed by $100\%$ of the \textit{Medium} and $X\%$ of the \textit{Medium-Replay}.
The results in Fig~\ref{fig-mix-data} show that FOCUS still performs well with a small ratio of \textit{Medium-Replay} data while the baseline performs well only with a big ratio, which indicates that FOCUS is less dependent on the diversity of data.

\subsection{Ablation Study}
To evaluate the contribution of each component, we perform an ablation study for FOCUS. 
The results in Table~\ref{table-cd} show that the KCI test and our principle of choosing conditional variables contribute to the causal structure learning of both accuracy and robustness.

\section{Conclusion}
In this paper, we provide theoretical evidence about the advantages of using a causal world-model in offline RL.
We present error bounds of model prediction and policy evaluation in offline MBRL with causal and plain world-models.
We also propose a practical algorithm, FOCUS, to address the problem of learning causal structure in offline RL.
The main idea of FOCUS is to utilize causal discovery methods for offline causal structure learning.
We design a general mechanism to solve problems in extending causal discovery methods in RL, which includes conditional variables choosing.
Extensive experiments on the typical benchmark demonstrate that FOCUS achieves accurate and robust causal structure learning and thus significantly surpasses baselines in offline RL.

The limitation of FOCUS lies in: In our theory, we assume that the true causal structure is given. 
However, in practice, one needs to learn it from data and then use it. 
Quantifying the uncertainty in the learned causal structure from data is known to be a hard problem, and as one line of our future research, we will derive the generalization error bound with the causal structure learned from data.

\bibliography{neurips_2022}

\bibliographystyle{unsrt}

\newpage
\clearpage
\par
\newpage

\section*{Appendix}
\appendix
\section{Theory}
\label{app:theory}

\begin{definition}[Optimization objective in data distribution $\mathcal{D}$:]
\begin{align}
\min_{\beta} \mathbb{E}_{(\textbf{X},Y)\sim\mathcal{D}}[\textbf{X}\beta-Y]^2.
\end{align}
\end{definition}

\begin{definition}[Optimization objective in data $\mathcal{D}_{train}$:]
\begin{align}
    \min_{\beta} \mathbb{E}_{(\textbf{X},Y)\sim\mathcal{D}_{train}}[\textbf{X}\beta-Y]^2. 
\end{align}
\end{definition}

\begin{definition}[Optimization objective in data $\mathcal{D}_{train}$ with regularization:]
\begin{align}
    \min_{\beta} \mathbb{E}_{(\textbf{X},Y)\sim\mathcal{D}_{train}}[\textbf{X}\beta-Y]^2+k\|\beta\|^2,
\end{align}
\end{definition}

\begin{lemma}
Given that $\omega_{cau}\circ\beta^*$ is the optimal solution of Problem~\ref{problem-0},
suppose that in $D_{train}$, $\textbf{X}_{spu}=(\textbf{X}\circ\omega_{cau})\gamma_{spu}+\epsilon_{spu}$ where $\mathbb{E}_{D_{train}}[\epsilon_{spu}]=0$ and $\gamma_{spu}\neq\textbf{0}$, we have that $\hat{\beta}_{spu}\triangleq\omega_{cau}\circ(\beta^*-\lambda\gamma_{spu})+\lambda\omega_{spu}$ is also an optimal solution of Problem~\ref{problem-1} for any $\lambda$:
\begin{align*}
    &\mathbb{E}_{(\textbf{X},Y)\sim D_{train}}\Big[\Big(|\textbf{X}(\omega_{cau}\circ\beta^*)-Y|_2\Big)\mid \textbf{X}\Big]
    =\mathbb{E}_{(\textbf{X},Y)\sim D_{train}}\Big[\Big(|\textbf{X}\hat{\beta}_{spu}-Y|_2\Big)\mid \textbf{X}\Big]
\end{align*}
\end{lemma}
\begin{proof}
\begin{align*}
&\mathbb{E}_{(\textbf{X},Y)\sim   D_{train}}\Big[\Big(|(\textbf{X}\circ\omega_{cau})\beta^*-Y|_2\Big)\mid \textbf{X}\Big]\\
=&\mathbb{E}_{(\textbf{X},Y)\sim D_{train}}\Big\{\Big[|(\textbf{X}\circ\omega_{cau})(\beta^*-\lambda\gamma_{spu}+\lambda\gamma_{spu})
-Y|_2\Big]\mid \textbf{X}\Big\}\\
=&\mathbb{E}_{(\textbf{X},Y)\sim D_{train}}\Big\{\Big[|(\textbf{X}\circ\omega_{cau})(\beta^*-\lambda\gamma_{spu}) 
+(\textbf{X}\circ\omega_{cau})\lambda \gamma_{spu}-Y|_2\Big]\mid \textbf{X}\Big\}\\
=&\mathbb{E}_{(\textbf{X},Y)\sim D_{train}}\Big\{\Big[|(\textbf{X}\circ \omega_{cau})(\beta^*-\lambda\gamma_{spu})
+\lambda (X_{spu}-\epsilon_{spu})-Y|_2\Big]\mid \textbf{X}\Big\}\\
=&\mathbb{E}_{(\textbf{X},Y)\sim D_{train}}\Big\{\Big[|\textbf{X}(\omega_{cau}\circ(\beta^*-\lambda\gamma_{spu}))
+\lambda (\textbf{X}\circ\omega_{spu})-Y|_2\Big]\mid \textbf{X}\Big\}\\
 &\quad \quad\Big(\textrm{Since\ }\mathbb{E}_{(\textbf{X},Y)\sim D_{train}}[\epsilon_{spu}]=0\Big)\\
=&\mathbb{E}_{(\textbf{X},Y)\sim D_{train}}\Big\{\Big[|\textbf{X}(\omega_{cau}\circ(\beta^*-\lambda\gamma_{spu})+\lambda\omega_{spu})
-Y|_2\Big]\mid \textbf{X}\Big\}\\
=&\mathbb{E}_{(\textbf{X},Y)\sim D_{train}}\Big\{\Big[| X\hat{\beta}_{spu}-Y|_2\Big]\mid \textbf{X}\Big\} 
\quad \quad  \\
&\Big(\rm{Let} \  \hat{\beta}_{spu} \rm{denote} \ \ \omega_{cau}\circ(\beta^*-\lambda\gamma_{spu})+\lambda\omega_{spu}\Big)
\end{align*}
\end{proof}
\begin{lemma}[\textbf{$\lambda$ Lemma}]
Given $\lambda$ as the coefficient in Lemma~\ref{lemma-equal}, and $k$ in Problem~\ref{problem-2} chosen by Hoerl-Kennard formula, we have the solution of $\lambda$ in Problem~\ref{problem-2} that:
\begin{align}
\lambda
&=\frac{ \beta^*\gamma_{spu}}{{\beta^*}^2+ \gamma_{spu}^{2}+1+\frac{\sigma ^{2}_{spu}}{\sigma^2_{cau}}(1+\frac{1}{(\beta^*)^2})}
\end{align}
\end{lemma}
\begin{proof}
Since the solution of the ridge regression is $$\beta(k)=(\textbf{X}^T\textbf{X}+k\textbf{I})^{-1}\textbf{X}^T\textbf{Y},$$
we take $\hat{\beta}_{spu}$ into this solution and get: 
\begin{align}
\lambda=\frac{\sigma ^{2}_{cau} \beta^*\gamma_{spu}k}{\sigma ^{2}_{cau} \sigma ^{2}_{spu} +\sigma ^{2}_{cau} \gamma_{spu}^{2} k+\sigma ^{2}_{cau} k+\sigma ^{2}_{spu} k+k^{2}}
\end{align}
Since $k$ is chosen by  Hoerl-Kennard formula that $k=\frac{\sigma^2_{spu}}{(\beta^*)^2}$ , we have:
\begin{align*}
\lambda
&=\frac{\sigma ^{2}_{cau} \beta^*\gamma_{spu}}{\sigma ^{2}_{cau} \sigma ^{2}_{spu} /k+\sigma ^{2}_{cau} \gamma_{spu}^{2}+\sigma ^{2}_{cau}+\sigma ^{2}_{spu}+k}\\
&=\frac{\sigma ^{2}_{cau} \beta^*\gamma_{spu}}{\sigma ^{2}_{cau} \sigma ^{2}_{spu} /(\frac{\sigma^2_{spu}}{(\beta^*)^2})+\sigma ^{2}_{cau} \gamma_{spu}^{2}+\sigma ^{2}_{cau}+\sigma ^{2}_{spu}+\frac{\sigma^2_{spu}}{(\beta^*)^2}}\\
&=\frac{\sigma ^{2}_{cau} \beta^*\gamma_{spu}}{\sigma ^{2}_{cau}{\beta^*}^2+\sigma ^{2}_{cau} \gamma_{spu}^{2}+\sigma ^{2}_{cau}+\sigma ^{2}_{spu}+\frac{\sigma^2_{spu}}{(\beta^*)^2}}\\
&=\frac{\sigma ^{2}_{cau} \beta^*\gamma_{spu}}{\sigma ^{2}_{cau}({\beta^*}^2+ \gamma_{spu}^{2}+1)+\sigma ^{2}_{spu}+\frac{\sigma^2_{spu}}{(\beta^*)^2}}\\
&=\frac{ \beta^*\gamma_{spu}}{{\beta^*}^2+ \gamma_{spu}^{2}+1+\frac{\sigma ^{2}_{spu}}{\sigma^2_{cau}}(1+\frac{1}{(\beta^*)^2})}\\
\end{align*}
\end{proof}

\begin{proposition}
Given $\lambda$ as Formula~\ref{formula-lambda}, we have
$$
-\frac{1}{2}\leq\lambda\leq\frac{1}{2}.
$$
\end{proposition}

\begin{proof}
\begin{align*}
|\lambda|
&=\frac{| \beta^*\gamma_{spu}|}{|{\beta^*}^2+ \gamma_{spu}^{2}+1+\frac{\sigma ^{2}_{spu}}{\sigma^2_{cau}}(1+\frac{1}{(\beta^*)^2})|}\\
&\leq \frac{| \beta^*\gamma_{spu}|}{|{\beta^*}^2+ \gamma_{spu}^{2}+1|+|\frac{\sigma ^{2}_{spu}}{\sigma^2_{cau}}(1+\frac{1}{(\beta^*)^2})|}\\
&\leq \frac{| \beta^*\gamma_{spu}|}{|{\beta^*}^2+ \gamma_{spu}^{2}+1|}\\
&\leq \frac{| \beta^*\gamma_{spu}|}{|2{\beta^*}\gamma_{spu}+1|}\\
&\leq \frac{1}{2}
\end{align*}
So we have :  $-\frac{1}{2}\leq\lambda\leq\frac{1}{2}$.
\end{proof}

\begin{theorem}[\textbf{Spurious Theorem}]
Let $\mathcal{D}=\{(\textbf{X},Y)\}$ denote the data distribution, $\hat{\beta}_{spu}$ denote the solution in Lemma~\ref{lemma-equal} with $\lambda$ in Lemma~\ref{lambda-lemma},
and $\hat{Y}_{spu}=\textbf{X}\hat{\beta}_{spu}$ denote the prediction.
Suppose that the data value is bounded: $|X_i|_1 \leq X_{max}, i=1,\cdots,n$ and the error of optimal solution $\epsilon_{cau}$ is also bounded: $|\epsilon_{cau}|_1\leq\epsilon_c$, we have the model prediction error bound:
\begin{align}
\mathbb{E}_{(\textbf{X},Y)\sim D}[(|\hat{Y}_{spu}-Y|_1)\mid \textbf{X}]\leq X_{max}|\lambda|_1(|\gamma_{spu}|_1+1)+\epsilon_{c}.
\end{align}
\end{theorem}

\begin{proof}

Let $\hat{Y}_{cau}$ denote $(\textbf{X}\circ\omega_{cau})\beta^*$, we have
\begin{align*}
&\mathbb{E}_{(\textbf{X},Y)\sim D_{test}}\Big[\Big(|\hat{Y}_{spu}-Y|_1\Big)\mid \textbf{X}\Big]\\
=&\mathbb{E}_{(\textbf{X},Y)\sim D_{test}}\Big[\Big(|(\hat{Y}_{spu}-\hat{Y}_{cau})+(\hat{Y}_{cau}-Y)|_1\Big)\mid \textbf{X}\Big]\\
\leq&\mathbb{E}_{(\textbf{X},Y)\sim D_{test}}\Big[\Big(|(\hat{Y}_{spu}-\hat{Y}_{cau})|_1\Big)\mid \textbf{X}\Big] 
+\mathbb{E}_{(\textbf{X},Y)\sim D_{test}}\Big[\Big(|(\hat{Y}_{cau}-Y)|_1\Big)\mid \textbf{X}\Big]\\
\leq&\mathbb{E}_{(\textbf{X},Y)\sim D_{test}}\Big[\Big(|\textbf{X}\lambda(-\omega_{cau}\circ\gamma_{spu}+\omega_{spu})|_1\Big)\mid \textbf{X}\Big]+\epsilon_{c}\\
=&\mathbb{E}_{(\textbf{X},Y)\sim D_{test}}\Big[\Big(|\textbf{X}\lambda (-\gamma_{spu}+\omega_{spu})|_1\Big)\mid \textbf{X}\Big]+\epsilon_{c}\\
\leq&\mathbb{E}_{(\textbf{X},Y)\sim D_{test}}\Big[\Big(X_{max}|\lambda|_1 *|-\gamma_{spu}+\omega_{spu}|_1\Big)\mid \textbf{X}\Big]+\epsilon_{c}\\
\leq&\mathbb{E}_{(\textbf{X},Y)\sim D_{test}}\Big[\Big(X_{max}|\lambda|_1*(|\gamma_{spu}|_1+1)\Big)\mid \textbf{X}\Big]+\epsilon_{c}\\
=&X_{max}|\lambda|_1(|\gamma_{spu}|_1+1)+\epsilon_{c}
\end{align*}
\end{proof}

\begin{theorem}[\textbf{RL Spurious Theorem}]
Given an MDP with the state dimension $n_s$ and the action dimension $n_a$, 
a data-collecting policy $\pi_D$,
let $M^*$ denote the true transition model, 
$M_\theta$ denote the learned model that $M_\theta^i$ predicts the $i^{th}$ dimension with spurious variable sets $spu_i$ and causal variables $cau_i$, i.e.,  $\hat{S}_{t+1,i}=M^i_{\theta}((\textbf{S}_t,\textbf{A}_t)\circ \omega_{cau_i\cup spu_i})$.
Let $V_{\pi}^{M_\theta}$ denote the policy value of the policy $\pi$ in model $M_{\theta}$ and correspondingly $V_{\pi}^{M^*}$.
For any bounded divergence policy $\pi$, i.e. $\max_{S} D_{KL}(\pi(\cdot|S),\pi_{D}(\cdot|S))\leq \epsilon_{\pi}$, we have the policy evaluation error bound:
\begin{align}
|V_{\pi}^{M_\theta}-V_{\pi}^{M^*}|\leq &
\frac{2\sqrt{2}R_{max}}{(1-\gamma)^2}\sqrt{\epsilon_{\pi}}+ 
&\frac{R_{max}\gamma}{2(1-\gamma)^2}S_{max}[n_s\epsilon_c+
&(1+\gamma_{max})\lambda_{max}n_s(n_s+n_a)R_{spu}]
\end{align}
where $R_{spu}=\frac{\sum_{i=1}^{n_s}|spu_i|}{n_s(n_s+n_a)}$, which represents the spurious variable density, that is, the ratio of spurious variables in all input variables .
\end{theorem}

\begin{proof}

Before proving, we first introduce three lemmas:
\begin{lemma}
\begin{align*}
    |V_{\pi}^{M_\theta}-V_{\pi}^{M^*}|\leq& |V_{\pi}^{M^*}-V_{\pi_D}^{M^*}|+ |V_{\pi_D}^{M_\theta}-V_{\pi_D}^{M^*}|
    +|V_{\pi_D}^{M_\theta}-V_{\pi}^{M_\theta}| \\
     \leq&\frac{2\sqrt{2}R_{max}}{(1-\gamma)^2}\sqrt{\epsilon_\pi}+|V_{\pi_D}^{M_\theta}-V_{\pi_D}^{M^*}|
\end{align*}
\end{lemma}

\begin{lemma}
\begin{align*}
    |V_{\pi_D}^{M_{\theta}}-V_{\pi_D}^{M_{*}}|\leq \frac{R_{max}}{1-\gamma}\sum_s|d_{\pi_D}^{M_\theta}(s)-d_{\pi_D}^{M^*}(s)|\sum_a\pi_D(a|s)
\end{align*}
\end{lemma}

\begin{lemma}
\begin{align*}
    |d_{\pi_D}^{M_\theta}(s)-d_{\pi_D}^{M^*}(s)|
\leq& \frac{\gamma}{(1-\gamma)}\sum_{s,a,s'}|M_\theta(S_t,A_t)-M^*(S_t,A_t)| 
    \pi_D(a|s)d_{\pi_D}^{M^*}(s)
\end{align*}
\end{lemma}
The detailed proof of these lemmas can be found in \cite{DBLP:conf/nips/XuLY20}, which is omitted in this paper.
Based on the model prediction error bound in Theorem~\ref{theorem-spurious}, we have:
\begin{align*}
|M_\theta(S_t,A_t)-M^*(S_t,A_t)|
=&\sum_{i=1}^{n_s}|M_\theta^i(S_t,A_t)-M^{*,i}(S_t,A_t)| \\
\leq&\sum_{i=1}^{n_s} S_{max}[\epsilon_c+(\gamma_{max}+1)\lambda_{max}|spu_i|] \\
=&S_{max}[{n_s}\epsilon_c+(\gamma_{max}+1)\lambda_{max}\sum_{i=1}^{n_s}|spu_i|]\\
=&S_{max}[{n_s}\epsilon_c+(\gamma_{max}+1)\lambda_{max}n_s(n_s+n_a)R_{spu}]
\end{align*}
With above lemmas, we have: 
\begin{align*}
|V_{\pi}^{M_\theta}-V_{\pi}^{M^*}|\leq &
\frac{2\sqrt{2}R_{max}}{(1-\gamma)^2}\sqrt{\epsilon_{\pi}}+
&\frac{R_{max}\gamma}{2(1-\gamma)^2}S_{max}[n_s\epsilon_c+
&(\gamma_{max}+1)\lambda_{max}n_s(n_s+n_a)R_{spu}]
\end{align*}

\end{proof}
\section{Algorithm}
\subsection{Choosing the Threshold of p-value}
\label{app:alg-1}
To be fair, we share a common $p^*$ for the testing between any two variables.
The choice of $p^*$ significantly influences the accuracy of causal discovery that too small and too big both lead to causal misspecification.
The intuition behind our choosing principle is that there is a significant gap in the $p$ value between the causal relation and non-causal relation.
Based on this intuition, we partition the probability range $[0,1]$ into several intervals $[0,p_1),[p_1,p_2),\cdots,[p_n,1]$ according to the sorted $p$ values $\{p_i\}_{i=1}^n$ and design $p^*$ by the formula:
\begin{align}
p^*=\arg\max_{p_i} \frac{p_{i+1}}{i+1}-\frac{p_i}{i}.
\end{align}
If we only consider the biggest gap between $p_i$, then we will easily choose a big but improper $p^*$ due to the distribution of $p_i$ in some intervals (e.g., [0.5,1]) may be very sparse and thus leads to a big gap.

\begin{figure}[ht!]
\begin{center}
\subfigure[\textit{Random}]{
        \centering
        \includegraphics[width=0.3\linewidth]{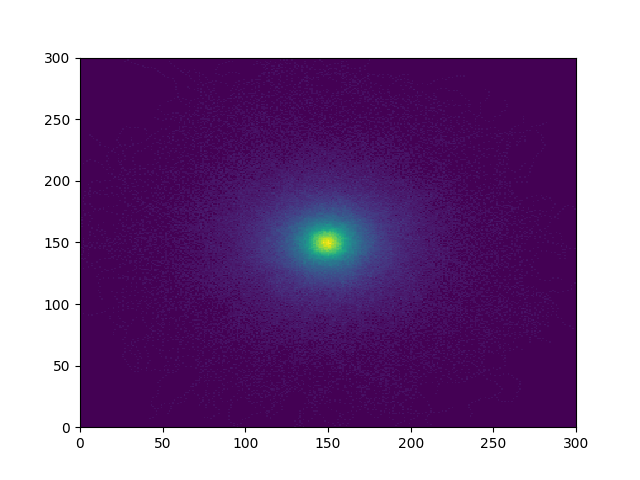} 
        }
\subfigure[\textit{Medium}]{
        \centering
        \includegraphics[width=0.3\linewidth]{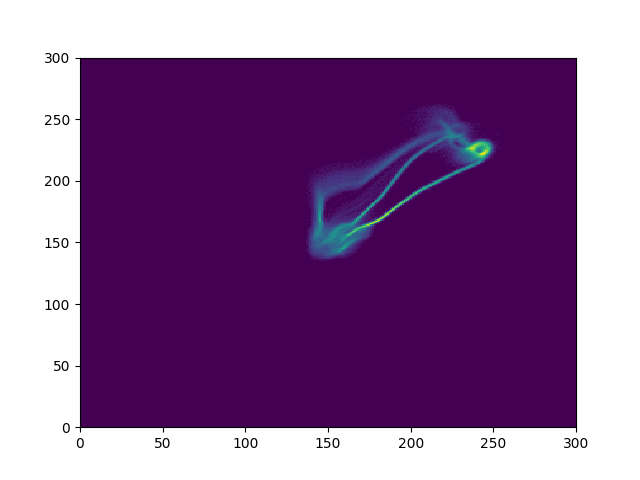} 
        }
\subfigure[\textit{Medium-Replay}]{
        \centering
        \includegraphics[width=0.3\linewidth]{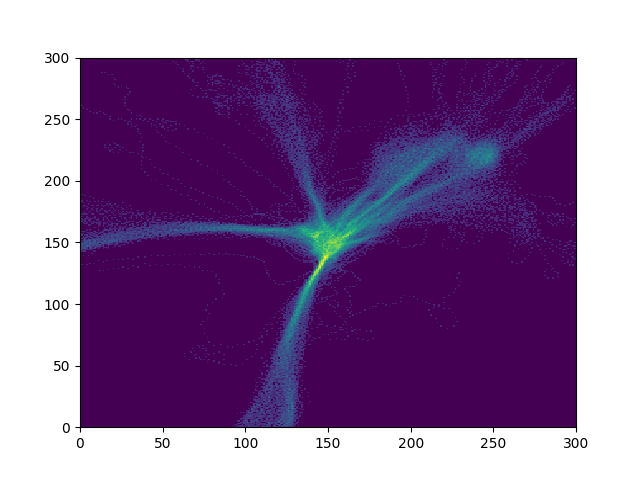} 
        }
\end{center}
\caption{The heat map of the three offline data sets. The high brightness represents high data density.}
\label{energy}
\end{figure}
\subsection{Causal Structure Network}
\label{app:alg-2}
The details of Causal Structure Network is shown in Algorithm~\ref{alg-network}. 

\begin{algorithm}[ht!]
  \caption{Causal Structure Network $\mathcal{M}_{Causal}(\cdot)$}
  \label{alg-network}
\begin{algorithmic}
  \STATE {\bfseries Input:} state $\textbf{s}_t\in\mathbb{R}^{n_s}$, action $\textbf{a}_t\in\mathbb{R}^{n_a}$, 
  \STATE causal structure mask matrix $G\in\{0,1\}^{(n_s,n_a)\times n_s}$,
  \STATE Make $\mathcal{M}_{i}(\cdot;\theta_i)$ as the copy of the basic model $\mathcal{M}(\cdot;\theta)$, where $i=1,\cdots, n_s$.
  \FOR {$i=1$ {\bfseries to} $n_s$}
  \STATE {Let $G_{\cdot,i}$ denote the $i^{th}$ column of $G$}
  \STATE {Get the masked input $X=(\textbf{s}_t,\textbf{a}_t)\circ G_{\cdot,i}$}
  \STATE {Get prediction $\tilde{Y}=\mathcal{M}_{i}(X;\theta_i)\in\mathbb{R}^{n_s}$}
  \STATE {Let $Y_{i}$ denote the $i^{th}$ element of $\tilde{Y}$.}
  \ENDFOR
  \STATE {\bfseries Return} $Y=(Y_i)_{i=1}^{n_s}$.
\end{algorithmic}
\end{algorithm}
\section{Experiments}
\subsection{Environment Details}
\label{app:exp-1}

The heat map of the data diversity is shown in Fig~\ref{energy}. 
In \textit{Random}, the data is clustered around the origin.
In \textit{Medium}, the data is gathered on a fixed trajectory from the origin to the destination.
In \textit{Medium-Replay}, the data is much more diverse where a lot of unseen data in above data sets is also sampled. 

The visualization of the state in Car Driving and the ground truth of its causal graph are shown in Fig~\ref{fig-benchmark}.

MuJoCo formulates robot control into MDPs with discrete timestep via equal interval sampling of the continuous-time. 
Therefore, for each timestep $t$, $s_{t+1}$ is the result of numerous times of simulation based on $s_t$ with repeated action $a_t$. 
Even if spurious variables are existed in one time of simulation, after numerous simulations, the causal effect will be propagated to almost variables, which leads to a full-connection causal graph ($R_{spu}=0$).
Therefore FOCUS degrades into vanilla MOPO in this scenario, which is meaningless to test. 
Fortunately, after analyzing the propagate progress of the dynamics, we found that the \textit{Inverted Pendulum} is a special case where the causal graph will keep sparse after numerous simulations.

\subsection{Experiment Result Details}
\label{app:exp-2}

The detailed training curves are shown in Fig~\ref{fig-policy-return}.

\begin{figure}[h!]
\centering
\subfigure[\textit{Random}]{
        \centering
        \includegraphics[width=0.3\linewidth]{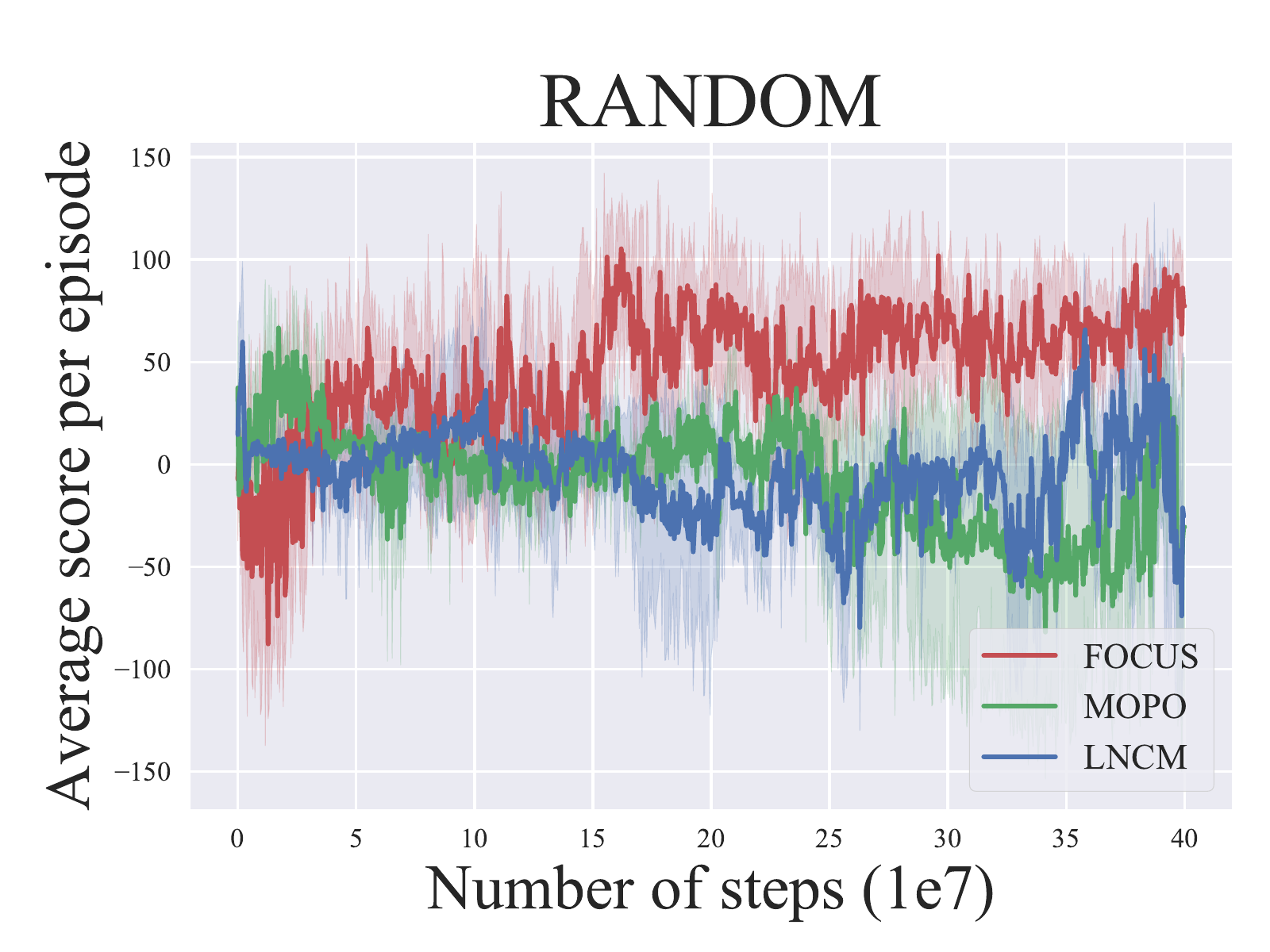} 
        }
\subfigure[\textit{Medium}]{
        \centering
        \includegraphics[width=0.3\linewidth]{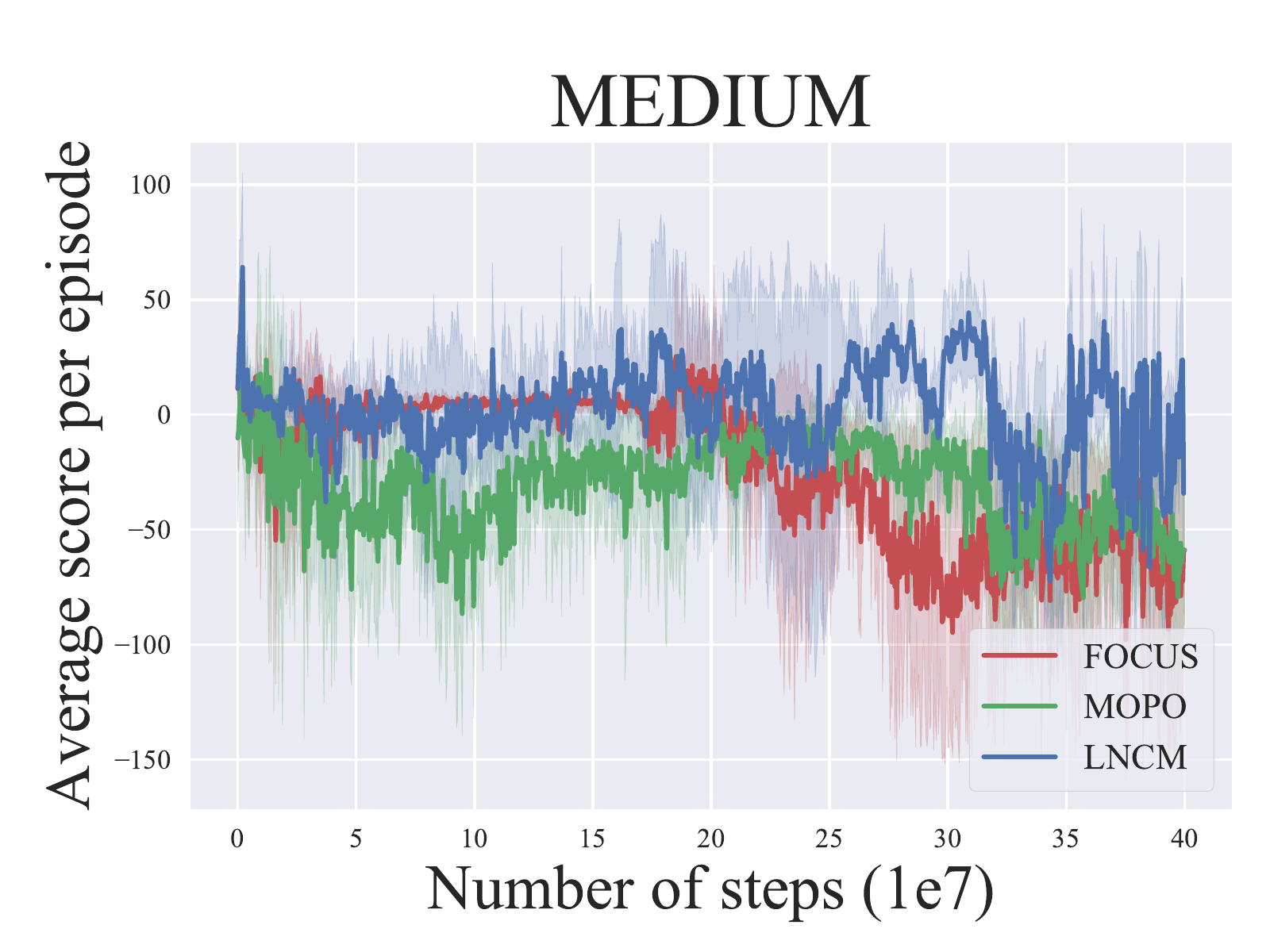}
        }
\subfigure[\textit{Medium-Replay}]{
        \centering
        \includegraphics[width=0.3\linewidth]{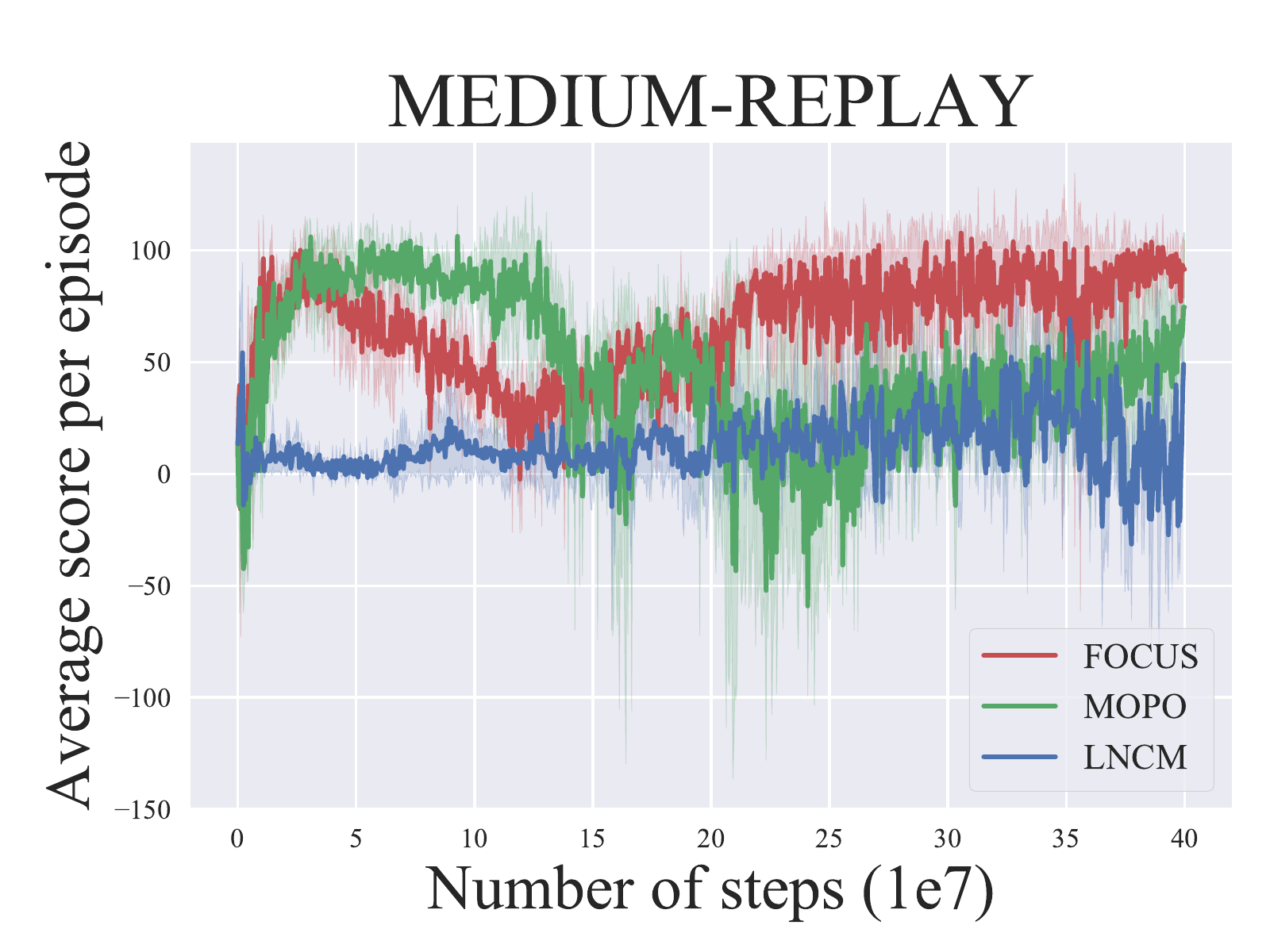}
        }
       \vskip -0.1in
       
\subfigure[\textit{Random}]{
        \centering
        \includegraphics[width=0.3\linewidth]{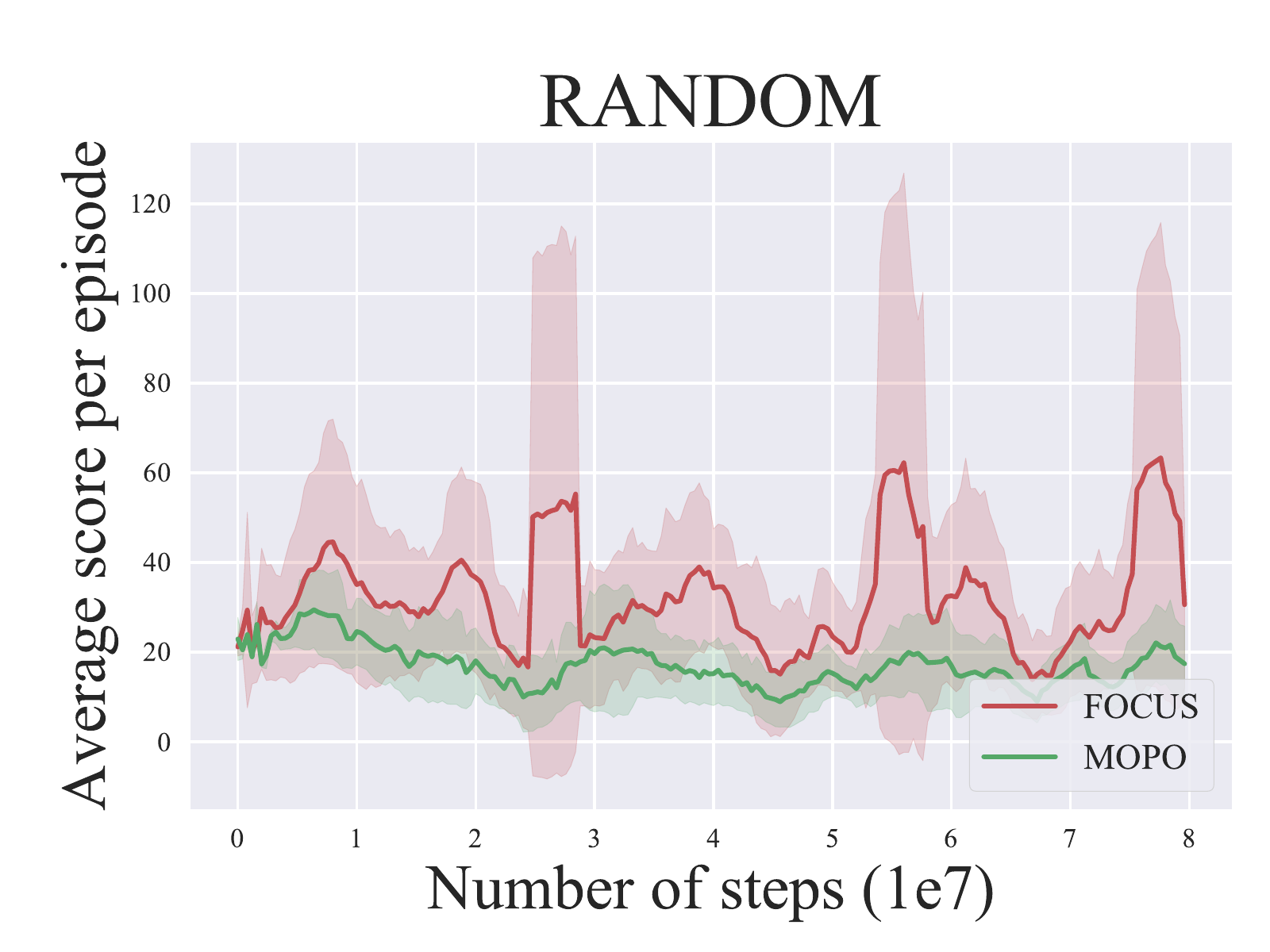}
        }
\subfigure[\textit{Medium}]{
        \centering
        \includegraphics[width=0.3\linewidth]{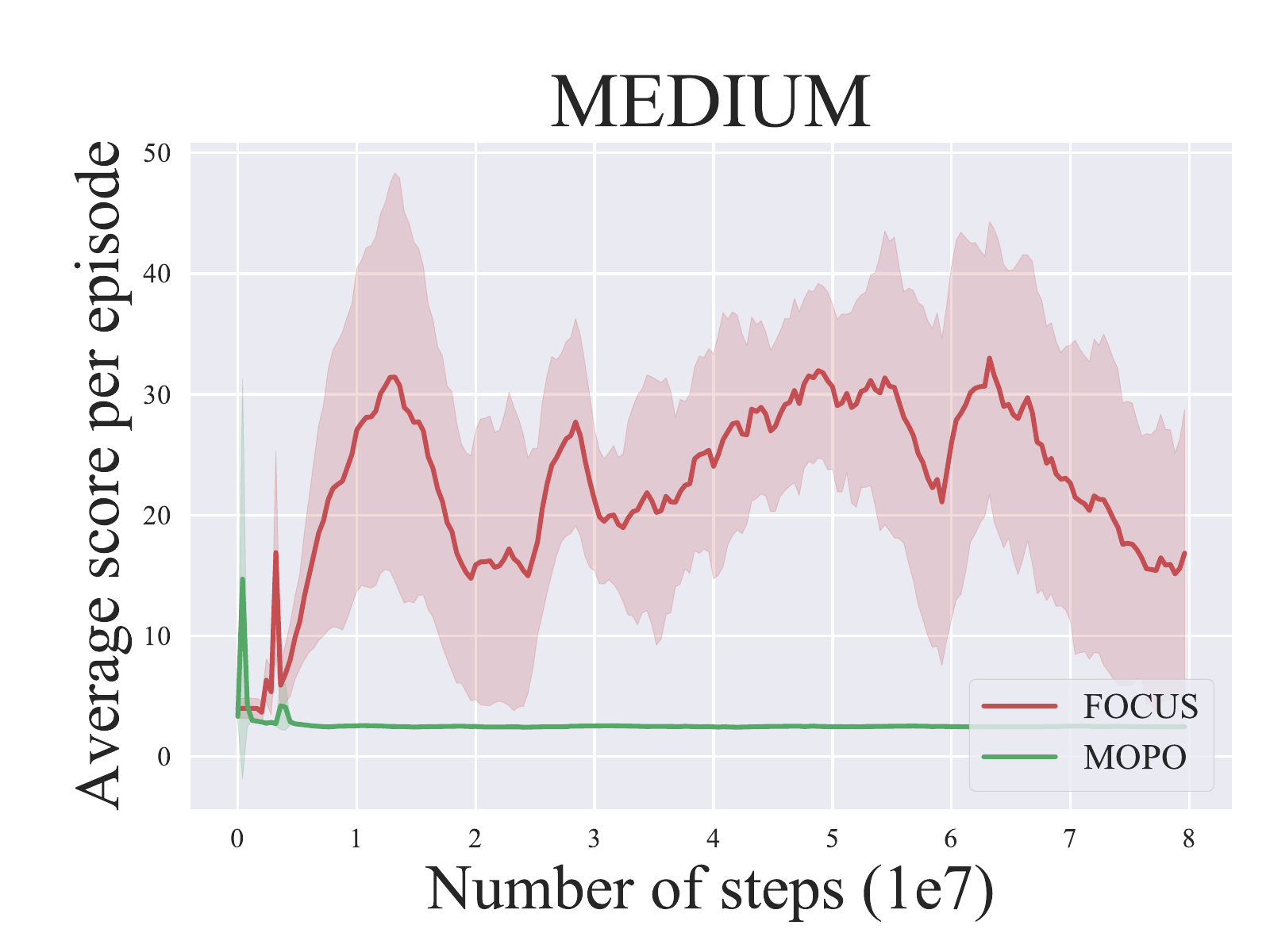}
        }
\subfigure[\textit{Medium-Replay}]{
        \centering
        \includegraphics[width=0.3\linewidth]{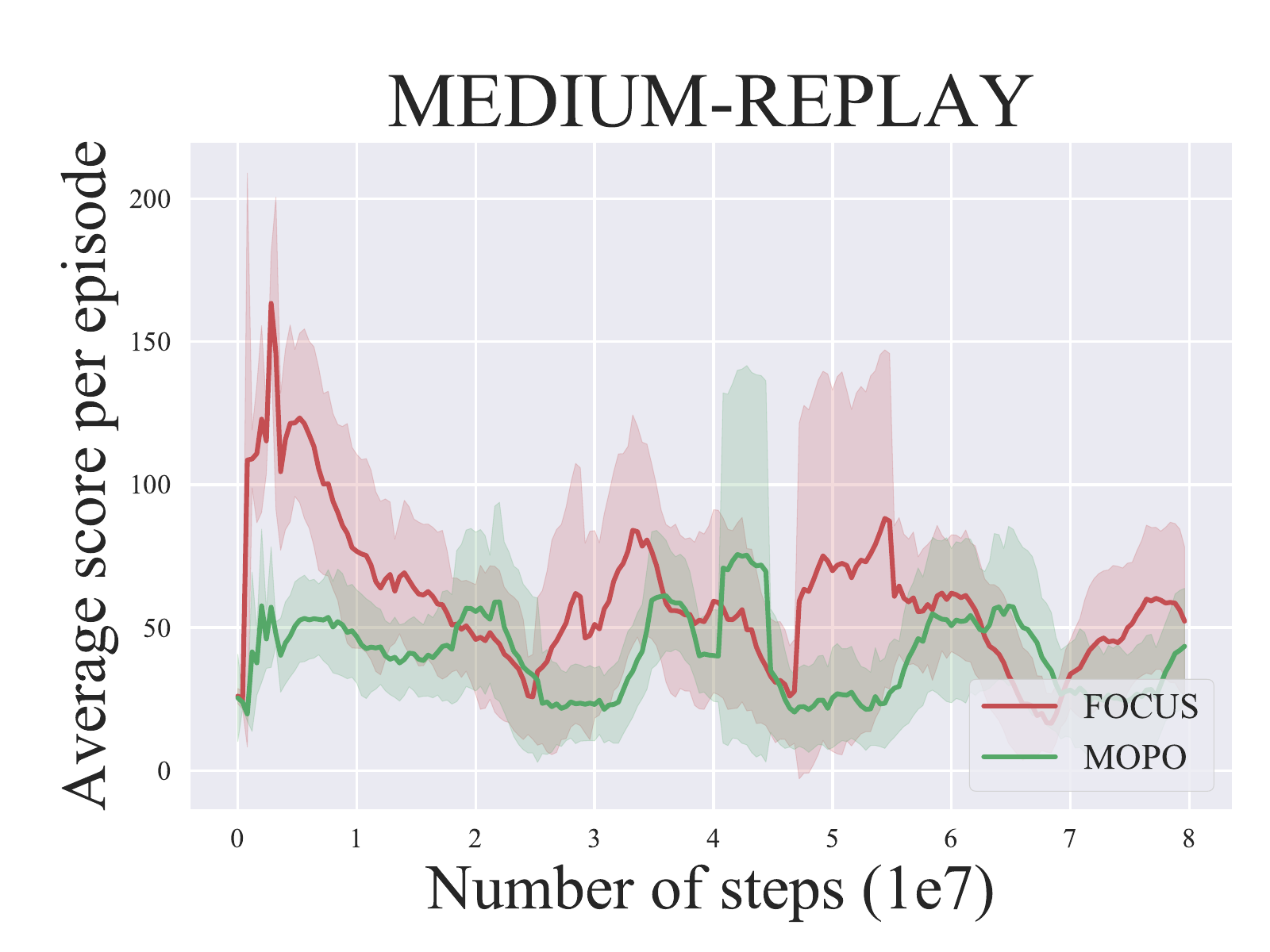}
        }
\vskip -0.1in
    \caption{Comparison of FOCUS and the baselines in the two benchmarks. 
    \textbf{(a)-(c):} The comparison in the Car Driving on the three datasets. 
    \textbf{(d)-(f):} The comparison in the Inverted Pendulum of MuJoCo on the three datasets.  
    }
   
    \label{fig-policy-return}

\end{figure}

\begin{figure}[h!]
\vskip -0.1in
\begin{center}
\begin{minipage}[b]{0.57\columnwidth}
\centerline{\includegraphics[width=0.9\columnwidth]{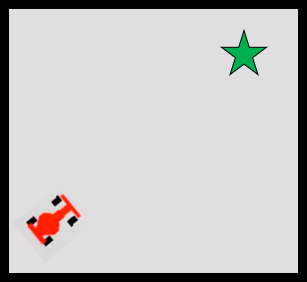}}
\end{minipage}
\begin{minipage}[b]{0.33\columnwidth}
\centerline{\includegraphics[width=0.9\columnwidth]{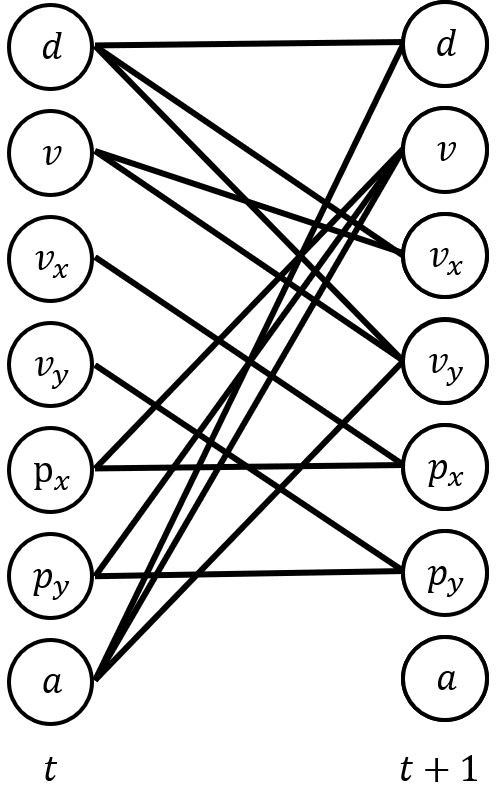}}
\end{minipage}
\caption{The visualization of the state and the causal structure for the Car Driving benchmark.
\textbf{Left:} the Toy Car Driving. The goal of the agent is to arrive at the star-shape destination.
\textbf{Right:} The ground truth of the causal structure in Toy Car Driving. The state is vector-based and its value is continuous.
}
\label{fig-benchmark}
\end{center}
\end{figure}

\end{document}